%% file: main.tex
\newif\ifarxiv
\arxivtrue
\documentclass[letterpaper]{article} 
\usepackage{aaai24}  
\usepackage{times}  
\usepackage{helvet}  
\usepackage{courier}  
\usepackage[hyphens]{url}  
\usepackage{graphicx} 
\usepackage{natbib}  
\usepackage{caption} 
\usepackage{algorithm}
\usepackage{algorithmic}

\usepackage{newfloat}
\usepackage{listings}
\DeclareCaptionStyle{ruled}{labelfont=normalfont,labelsep=colon,strut=off} 
\floatstyle{ruled}
\newfloat{listing}{tb}{lst}{}
\floatname{listing}{Listing}
\title{s-ID: Causal Effect Identification in a Sub-Population}
\author{
    Amir Mohammad Abouei\thanks{Equal contribution.}\textsuperscript{\rm 1},
    Ehsan Mokhtarian\footnotemark[1]\textsuperscript{\rm 1},
    Negar Kiyavash\textsuperscript{\rm 1 2}
}
\affiliations{
    \textsuperscript{\rm 1}School of Computer and Communication Sciences, EPFL, Lausanne, Switzerland\\
    \textsuperscript{\rm 2}College of Management of Technology, EPFL, Lausanne, Switzerland\\
    \{amir.abouei, ehsan.mokhtarian, negar.kiyavash\}@epfl.ch
}

\usepackage{bibentry}

\usepackage{url}            
\usepackage{booktabs}       
\usepackage{amsfonts}       
\usepackage{nicefrac}       
\usepackage{microtype}      
\usepackage{xcolor}         
\usepackage{amsmath,amssymb,mathtools}
\usepackage{amsthm} 
\usepackage{tikz}
\usetikzlibrary{positioning}
\usetikzlibrary{shapes, arrows, arrows.meta, positioning}
\usepackage{caption}
\usepackage{subcaption}
\usepackage{mleftright}
\usepackage{graphicx}
\usepackage{multirow}
\usepackage[inline]{enumitem}

\usetikzlibrary{backgrounds, fit, positioning}
\tikzstyle{block} = [draw, fill=white, circle, text centered,inner sep=0.25cm]
\tikzstyle{block-s} = [draw, fill=white, circle, double, double distance=1pt, text centered,inner sep=0.2cm]

\input{commands}

\ifarxiv
\fi

\begin{document}

\maketitle

\begin{abstract}
    Causal inference in a sub-population involves identifying the causal effect of an intervention on a specific subgroup, which is distinguished from the whole population through the influence of systematic biases in the sampling process.
    However, ignoring the subtleties introduced by sub-populations can either lead to erroneous inference or limit the applicability of existing methods. We introduce and advocate for a causal inference problem in sub-populations (henceforth called \s-ID), in which we merely have access to observational data of the targeted sub-population (as opposed to the entire population). Existing inference problems in sub-populations operate on the premise that the given data distributions originate from the entire population, thus, cannot tackle the \s-ID problem. To address this gap, we provide necessary and sufficient conditions that must hold in the causal graph for a causal effect in a sub-population to be identifiable from the observational distribution of that sub-population. Given these conditions, we present a sound and complete algorithm for the \s-ID problem.

\end{abstract}

\section{Introduction}
    In machine learning, variable(s) $\Yb$ are commonly predicted from observed variable(s) $\Xb$ by estimating the conditional probability distribution $\pro(\Yb \vert \Xb)$ \cite{bishop2006pattern}.
    This approach is effective for understanding correlations or associations in the data, but it falls short when we seek to understand how changes in $\Xb$ would affect $\Yb$.
    Such an understanding requires a different methodology, known as causal inference (in population), which involves estimating the \emph{interventional distributions (or causal effect)}, denoted by $\pro_{\Xb}(\Yb)$.
    $\pro_{\Xb}(\Yb)$ represents the probability of an outcome $\Yb$ if we were to intervene or change the values of the input variable(s) $\Xb$ \cite{pearl2000models, pearl2009causality, hernan2010causal}.\footnote{We utilize Judea Pearl's framework for causal inference to present our findings.
    Within this framework, alternative notations for interventional distributions include $\pro(\mathbf{Y} \vert \text{do}(\mathbf{X}))$ and $\pro_{\text{do}(\mathbf{X})}(\mathbf{Y})$, which employ the $do()$ operator to denote an intervention.
    Nevertheless, for the sake of simplicity in notation, we adopt the latter representation and drop the $do()$.}
    
    The gold standard for estimating a causal effect is to perform experiments/interventions in the environment, for instance, by using techniques such as randomized controlled trials (RCTs) \cite{fisher1936design}.
    However, these methods often require real-world experiments, which can be prohibitively expensive, unethical, or simply infeasible in many scenarios.
    Alternatively, researchers can turn to observational methods, utilizing the \emph{causal graph} of the environment and available data to estimate interventional distributions \cite{pearl2009causality, spirtes2000causation}.
    The causal graph, a graphical representation that depicts the causal relationships between variables, plays a central role in this methodology.
    This observational approach avoids the need for costly or impractical experiments but comes with its own challenges.
    In particular, computing interventional distributions uniquely may not always be feasible.
    
    \textbf{Identifiablity in population.}
    Identifiability refers to the ability to uniquely compute a distribution from the available data.
    When all variables in the system are observable and the causal graph is known, all interventional distributions are identifiable using the so-called back-door adjustment sets, meaning all causal effects are identifiable \cite{pearl1993bayesian}.
    However, only a subset of causal effects can be identified in the presence of unobserved variables or hidden confounders \cite{pearl1995causal}.
    Selection bias can also make some causal effects unidentifiable \cite{shpitser2012identification}.
    This bias, which is similar to distribution mismatch in learning theory \cite{masiha2021learning}, often arises from conditioning on selection variables.
    The problem of causal effect identification in population pertains to whether, given the causal graph, an interventional distribution can be uniquely computed from the available data.
    Various forms of available data lead to different problems in causal inference in population, the most well-known of which is the ID problem \cite{pearl1995causal, tian2003ID}.
    This problem arises when the available data is from the joint distribution of the observed variables.
    A summary of these problems is provided in Table \ref{table:1}, and a more comprehensive discussion can be found in the Related Work section.
    
    \begin{table*}[t]
        \centering
        \begin{tabular}{c| c | c |c} 
            \toprule
            \multicolumn{2}{c|}{Causal inference problem} & Given distribution(s) & Target interventional distribution\\
            \hline \hline
            \multirow{3}{*}{On Population}
            & ID & $\pro (\Vb)~$ & $\inter$  \\
            & \s-Recoverability &  $\pro (\Vb \vert S = 1)$  &  $\inter$\\
            & gID & $\{\pro_{\Zb_i}(\Vb \setminus \Zb_i)\}_{i = 0}^{m}$ &  $\inter$ \\
            \hline
            \multirow{3}{*}{On Sub-Population}
            & c-ID  & $\pro (\Vb)$  &  $\pro_{\Xb} (\Yb \vert \Zb)$ \\
            & c-gID  & $\{\pro_{\Zb_i}(\Vb \setminus \Zb_i)\}_{i = 0}^{m}$ &  $\pro_{\Xb} (\Yb \vert \Zb)$ \\
            & \textbf{\s-ID}  &  $\pro (\Vb \vert S = 1)$  &   $\interSb$ \\
            \bottomrule
        \end{tabular}
        \caption{Various causal inference problems based on given and target distributions. 
        Herein, $\Vb$ is the set of observed variables, $\Xb$ is the set of intervened variables, $\Yb$ is the set of outcome variables, and $S=1$ corresponds to a sub-population. 
        In this paper, we introduce the \s-ID problem. Note that in all of these problems, the causal graph is given.
        }
        \label{table:1}
    \end{table*}
    
    \textbf{Conditional causal effects} represent the conditional distributions that capture the impact of a treatment on the outcome within specific contexts or sub-populations.
    This concept allows for targeted interventions and tailored policies, offering valuable insights for practical applications \cite{qian2011performance}.
    \citet{shpitser2012identification} considered the c-ID problem, which pertains to identifying a conditional interventional distribution $\pro_{\Xb} (\Yb \vert \Zb)$ from the joint distribution of observed variables.\footnote{In the notation $\pro_{\Xb} (\Yb \vert \Zb)$, it is important to note the sequence of operations. The notation signifies that we first intervene on the set $\Xb$ and then, within the resulting distribution, condition on $\Zb$.}
    An important practical limitation of the c-ID formulation is that it assumes access to samples from the observational distribution of the \textit{entire population} rather than just the target sub-population.
    Unfortunately, the c-ID identification result cannot be directly extended to the setting where the available samples are from the target sub-population, which is often the prevailing scenario in practical applications.
    The recent extension of c-ID, known as c-gID problem \cite{correa2021nested, kivva2023identifiability}, which we will discuss in Related Work, also suffers from the same practical limitation.
    
    \textbf{Identifiablity in sub-populations.}
    As mentioned earlier, a sub-population is a specific subset of individuals within a larger population distinguished by certain characteristics or traits.\footnote{A sample in a sub-population is generated from a conditional distribution that determines the characteristics of the sub-population. This often introduces selection bias, as the sampling process might not be representative of the entire population.}
    We utilize an auxiliary binary variable $S$ to model a sub-population akin to \citet{Bareinboim_Tian_2015}:
    $S$ is added as a child variable representing the specific traits that distinguish the sub-population of the population ($S$ can have several parents), and $S=1$ corresponds to the target sub-population.
    We will formally introduce the auxiliary variable $S$ in Equation \eqref{eq: def S}.
    In this paper, we address the problem of causal inference in a sub-population, where the objective is to identify $\interSb$, which is the causal effect of a treatment or intervention on a specific subgroup of individuals within a larger population.
    Specifically, we introduce the \emph{\s-ID} problem, an identification problem on sub-population when we merely have access to observational data of the target sub-population.
    That is, given the causal graph, we seek to determine when $\interSb$ can be uniquely computed from $\pro (\Vb \vert S = 1)$, where $\Vb$ is the set of observed variables.
    
    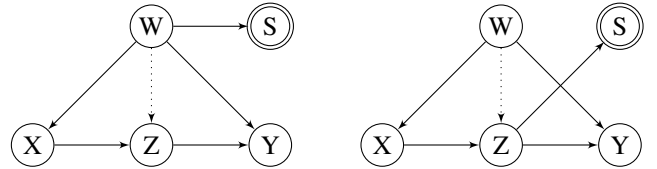
\begin{figure}[t] 
        \centering
        \tikzstyle{block} = [draw, fill=white, circle, text centered,inner sep= 0.2cm]
        \tikzstyle{input} = [coordinate]
        \tikzstyle{output} = [coordinate]
        \begin{tikzpicture}[->, auto, node distance=1.3cm,>=latex', every node/.style={inner sep=0.12cm}]
            \node [block, label=center:X](X) {};
            \node [block, right= 1cm of X, label=center:Z](Z) {};
            \node [block, right= 1cm of Z, label=center:Y](Y) {};
            \node [block, above= 1cm of Z, label=center:W](W) {};
            \node [block-s, right= 1cm of W, label=center:S](S) {};
            
            \draw (X) to (Z);
            \draw (Z) to (Y);
            \draw (W) to (X);
            \draw (W) to (Y);
            \draw (W) to (S);
            \draw[dotted] (W) to (Z);
        \end{tikzpicture}
        \hfill
         \begin{tikzpicture}[->, auto, node distance=1.3cm,>=latex', every node/.style={inner sep=0.12cm}]
            \node [block, label=center:X](X) {};
            \node [block, right= 1cm of X, label=center:Z](Z) {};
            \node [block, right= 1cm of Z, label=center:Y](Y) {};
            \node [block, above= 1cm of Z, label=center:W](W) {};
            \node [block-s, right= 1cm of W, label=center:S](S) {};
            
            \draw (X) to (Z);
            \draw (Z) to (Y);
            \draw (W) to (X);
            \draw (W) to (Y);
            \draw (Z) to (S);
            \draw[dotted] (W) to (Z);
        \end{tikzpicture}
        \caption{$X$: whether the public health policy bans smoking in public areas. $Y$: rate of lung cancer. $Z$: percentage of people who smoke. $W$: the average age of people.
        In the left causal graph, $\pro_X(Y \vert S=1)$ is \s-ID, i.e., can be computed from $\pro (X,Y,Z,W \vert S = 1)$, while it is not \s-ID in the right causal graph.}
        \label{fig: example intro}
    \end{figure}

    \textbf{A real-world example.}
    Consider the causal graphs depicted in Figure \ref{fig: example intro}, where we analyze a hypothetical scenario in a random country.
    Here:
    \begin{itemize}
        \item The treatment variable $X$ denotes whether smoking is banned in public areas.
        \item The mediator variable $Z$ indicates the percentage of the population that smokes.
        \item The outcome variable $Y$ measures the rate of lung cancer.
        \item The confounder variable $W$ captures the average age of the population.
    \end{itemize}
    
    Clearly $X$ influences $Z$, and both $Z$ and $W$ affect $Y$.
    The relationship between $W$ and $X$ can be explained by the possibility that in countries with older populations, there may be greater awareness and concern about the health risks of smoking, potentially leading to stricter health policies such as public smoking bans.
    Additionally, one could argue that $W$ may also have an impact on $Z$.
    Nevertheless, our subsequent analysis remains valid whether or not we consider a causal link between $W$ and $Z$.
    Now, consider the scenario where the data from $X, Y, Z, W$ is available from a subset of countries (sub-population) with younger populations than the world average.
    This scenario is illustrated in the left graph in Figure \ref{fig: example intro}.
    The \s-ID problem aims to identify the causal effect of a new policy $X$ on the outcome variable $Y$ for this target sub-population, given only observational data from this group.
    As we will demonstrate, this causal effect is \textit{identifiable} and can be calculated using Algorithm \ref{algo: SID}.
    
    In contrast, in the setting of the S-Recoverability problem, a causal inference problem in population (refer to the second row of Table \ref{table:1}), the task is to compute the causal effect of $X$ on $Y$ for the entire population using only data from this sub-population.
    The limitation of data coming only from the sub-population renders the inference for the whole population particularly challenging.
    Accordingly, \citet{Bareinboim_Tian_2015} showed that in this example, the causal effect of $X$ on $Y$ (in population) is unidentifiable.
    In the c-ID setting, the conditional causal effect of $X$ on $Y$ in sub-population is identifiable, but it requires observational data from the entire population, i.e., from all the countries in the world.
    Lastly, consider another scenario where the sub-population is based on a condition on the mediator variable $Z$ rather than the confounder $W$ (the right graph in Figure \ref{fig: example intro}).
    An example of this scenario might involve a sub-population of countries that have had high smoking rates in recent years.
    Applying our Theorem \ref{th:markov-single}, we can show that in this case, $\pro_X(Y \vert S=1)$ is \textit{not identifiable} from $\pro (\Vb \vert S = 1)$.
    Note that in the ID setting, $\pro_X(Y)$ is identifiable from $\pro (\Vb)$.
    This shows that simply ignoring the sub-population and applying any algorithms in the ID setting leads to an erroneous inference.
    
    The purpose of this example is to (i) demonstrate the critical role of causal graphs in whether a causal effect in a sub-population is identifiable or not and (ii) show that previous identification results in the literature do not suffice to answer the \s-ID problem.
    An additional example is provided in Appendix A.
    
    Our main contributions are as follows.
    \begin{itemize}
        \item We formally introduce the \s-ID problem, a practical scenario for causal inference in a sub-population. This problem asks whether, given a causal graph, a causal effect in a sub-population can be uniquely computed from the observational distribution pertaining to that sub-population.
        \item We provide necessary and sufficient conditions on the causal graph for when a causal effect in a sub-population can be uniquely computed from the observational distribution of the same sub-population (Theorems \ref{th:markov-single} and \ref{th:markov-general}).
        \item We propose a sound and complete algorithm for the \s-ID problem (Algorithm \ref{algo: SID}).
    \end{itemize}

\section{Preliminaries}
    Throughout the paper, we denote random variables by capital letters and sets of variables by bold letters.
    We use $\sum_{\Xb}$ to denote marginalization, i.e., summation (or integration for continuous variables) over all the realizations of the variables in a set $\Xb$.
    
    Let $\Gr$ be a directed acyclic graph (DAG) over a set of variables $\Vb$.
    We denote by $\Pa{X}{\Gr}$, $\Ch{X}{\Gr}$, and $\Anc{X}{\Gr}$ the set of parents, children, and ancestors of $X$ (including $X$) in $\Gr$, respectively.
    We further define $\Anc{\Xb}{\Gr}=\bigcup_{X \in \Xb}\Anc{X}{\Gr}$ for a set $\Xb \subseteq \Vb$.
    A structural equation model (SEM) describes the dynamics of a system using a collection of equations
    \begin{equation*}
        X = f_X(\Pa{X}{\Gr}, \varepsilon_X), \quad \forall X \in \Vb,
    \end{equation*}
    where $\Gr$ is the causal graph, $f_X$ is a deterministic function and $\varepsilon_X$ is the exogenous noise of variable $X$, which is independent of all the other exogenous noises.
    A SEM $\M$ with causal DAG $\Gr$ induces a unique joint distribution $\pro^{\M}(\Vb)$ that satisfies Markov property with respect to $\Gr$.
    That is, $\pro^{\M}(\Vb)$ can be factorized according to $\Gr$ as
    \begin{equation*}
        \pro^{\M}(\Vb) = \prod_{X\in \Vb} \pro^{\M}(X \vert \Pa{X}{\Gr}). 
    \end{equation*}
    We drop $\M$ from $\pro^{\M}(\cdot)$ when it is clear from the context.
    In this paper, an intervention on a set $\Xb$ converts $\M$ to a new SEM where the equations of the variables in $\Xb$ are replaced by some constants.\footnote{There are other types of interventions, such as soft-interventions, which are not considered herein.}
    We denote the corresponding post-interventional distribution by $\pro_{\Xb}(\Vb \setminus \Xb)$.
    The goal of causal inference in population is to compute an interventional distribution $\pro_{\Xb}(\Yb)$ for two disjoint subsets $\Xb$ and $\Yb$ of $\Vb$.
    
    Let $\Xb,\Yb,\Wb$ be three disjoint subsets of $\Vb$.
    A path $\mathcal{P}= (X, Z_1, \dots, Z_k, Y)$ between $X \in \Xb$ and $Y \in \Yb$ in $\Gr$ is called \emph{blocked} by $\Wb$ if there exists $1 \leq i \leq k$ such that
    \begin{itemize}
        \item $Z_i$ is a collider\footnote{A non-endpoint vertex on a path is called a collider if both of the edges incident to it on the path point to it.} on $\mathcal{P}$ and $Z_i \notin \Anc{\Wb}{\Gr}$, or
        \item $Z_i$ is not a collider on $\mathcal{P}$ and $Z_i \in \Wb$. 
    \end{itemize}
    Denoted by $(\Xb \independent \Yb \vert \Wb)_{\Gr}$, we say $\Wb$ $d$-separates $\Xb$ and $\Yb$ if for any $X \in \Xb$ and $Y \in \Yb$, $\Wb$ blocks all the paths in $\Gr$ between $X$ and $Y$.
    Conversely, $(\Xb \notindependent \Yb \vert \Wb)_{\Gr}$ if there exists at least one \emph{active path} between a variable in $\Xb$ and a variable in $\Yb$ that is not blocked by $\Wb$.
    
    The following three rules, commonly referred to as Pearl's do-calculus rules \cite{pearl2000models}, provide a tool for calculating interventional distributions using the causal graph.
    
    \begin{itemize}
        \item \textbf{Rule 1}: 
            If $(\Yb \independent \Zb \vert \Xb, \Wb)_{\Gr_{\overline{X}}}$, then
           \begin{equation*}
               \pro_{\Xb}(\Yb \vert \Zb, \Wb) = \pro_{\Xb}(\Yb \vert \Wb).
           \end{equation*}    
        \item \textbf{Rule 2}: 
            If $(\Yb \independent \Zb \vert \Xb, \Wb)_{\Gr_{\overline{\Xb}, \underline{\Zb}}}$, then
            \begin{equation*}
                  \pro_{\Xb, \Zb}(\Yb \vert \Wb) = \pro_{\Xb}(\Yb \vert \Zb, \Wb). 
            \end{equation*}  
        \item \textbf{Rule 3}: 
            If $(\Yb \independent \Zb \vert \Xb, \Wb)_{\Gr_{\overline{\Xb}, \overline{\Zb(W)}}}$, where $\Zb(\Wb) \coloneqq \Zb  \setminus \Anc{\Wb}{\Gr_{\overline{\Xb}}}$, then
            \begin{equation*}
                \pro_{\Xb, \Zb}(\Yb \vert \Wb) = \pro_{\Xb} (\Yb \vert \Wb). 
            \end{equation*}
    \end{itemize}
    
    In these rules, $\Gr_{\overline{\Xb} \underline{\Zb}}$ denotes the graph obtained by removing the incoming edges to $\Xb$ and outgoing edges from $\Zb$.

\section{The \s-ID Problem}
    In this section, we start by discussing the integration of an auxiliary variable $S$ into a SEM to model a sub-population and introduce the \s-ID problem and formulate our objective.
    Subsequently, we provide necessary and sufficient graph conditions for the \s-ID problem.
    Finally, given the causal DAG, we provide a sound and complete algorithm for computing $\pro_{\Xb}(\Yb \vert S=1)$ from $\pro (\Vb \vert S = 1)$, when this conditional causal effect is \s-ID.
    
    \subsection{Modeling a Sub-Population: Auxiliary Variable $S$}
        Let $\M$ be a SEM with the set of variables $\Vb$, causal DAG $\Gr$, and observational distribution $\pro(\Vb)$, representing the distribution of the entire \emph{population}.
        That is, a sample is from the population if it is generated from $\pro(\Vb)$.
        A \emph{sub-population}, on the other hand, refers to a biased sampling mechanism.
        Formally, a sample is from a sub-population if it is generated from a conditional distribution $\pro(\Vb \vert S=1)$, in which
        \begin{equation} \label{eq: def S}
            S \coloneqq f_S(\Vb_S, \varepsilon_S),
        \end{equation} 
        where $f_S$ is a binary function that determines the characteristics or traits of the sub-population, $\Vb_S \subseteq \Vb$, and $\varepsilon_S$ is the exogenous noise variable independent of the other exogenous variables.
        Under this modeling approach, $S=1$ signifies that the sample is generated from a specific sub-population.
    
        We denote by $\Grs$ the augmented DAG obtained by adding $S$ to $\Gr$, such that $\Pa{S}{\Grs} = \Vb_S$, and $S$ does not have any children.
        As a result, $\Grs$ is the causal graph of the SEM obtained by adding $S$ to the set of variables.
        Moreover, we define $\prs(\Vb) \coloneqq \pro(\Vb \vert S=1)$, which is the observational distribution of the target sub-population.
        We often omit the graph subscript in $\Pa{}{\Gr}$ and $\Anc{}{\Gr}$ notations as parents and ancestor sets are identical in $\Gr$ and $\Grs$.
    
    \subsection{Problem Formulation: Definition of \s-ID}
        As mentioned earlier, $\pro(\Vb \vert S=1)$ (or $\prs(\Vb)$) is the observational distribution of a sub-population.
        Furthermore, for two disjoint subsets $\Xb$ and $\Yb$ of $\Vb$, $\pro_{\Xb}(\Yb \vert S=1)$ (or $\prs_{\Xb}(\Yb)$) corresponds to the causal effect of $\Xb$ on $\Yb$ in that sub-population.
        The problem of \s-ID, formally defined in the following, considers the identifiability of $\pro_{\Xb}(\Yb \vert S=1)$ from $\pro(\Vb \vert S=1)$.
    
        \begin{definition}[\s-ID]\label{def:S-ID}
            Suppose $\Xb$ and $\Yb$ are disjoint subsets of a set $\Vb$, and let $\Grs$ be the augmented causal graph of a SEM over $\Vb \cup \{S\}$.
            Conditional causal effect $\pro_{\Xb}(\Yb \vert S = 1)$ is \s-ID in $\Grs$ if for any two SEMs $\M_1$ and $\M_2$ with causal graph $\Grs$ such that $\pro^{\M_1} (\Vb \vert S = 1) = \pro^{\M_2} (\Vb \vert S = 1) > 0$, then $\pro^{\M_1}_{\Xb} (\Yb \vert S = 1) = \pro^{\M_2}_{\Xb} (\Yb \vert S = 1)$.
        \end{definition}
    
        In other words, this definition states that $\prs_{\Xb}(\Yb)$ is \s-ID when it can be uniquely computed from $\prs(\Vb)$.
        
        In the rest of the paper, we address the following questions.
        Given an augmented causal DAG $\Grs$ over a set $\Vb \cup \{S\}$ and for two disjoint subsets $\Xb$ and $\Yb$ of $\Vb$,
        \begin{itemize}
            \item What are the necessary and sufficient conditions on $\Grs$ such that $\prs_{\Xb}(\Yb)$ is \s-ID in $\Grs$?
            \item When $\pro^{s}_{\Xb}(\Yb)$ is \s-ID in $\Grs$, how can we compute it from $\prs(\Vb)$?
        \end{itemize}
        To address the first question, for pedagogical reasons, we first consider the case where $\Xb$ and $\Yb$ each contain only one variable.
        We subsequently extend our findings to the multivariate scenario.
        In the last subsection, we address the second question and propose a sound and complete algorithm for the \s-ID problem.
    
    \subsection{Conditions for s-Identifiability: Singleton Case}
        Suppose $\Xb$ and $\Yb$ are singleton, where $\Xb = \{X\}$ and $\Yb = \{Y\}$.
        The following theorem provides a necessary and sufficient condition for $\prs_X(Y)$ to be \s-ID in $\Grs$.
    
        \begin{theorem}\label{th:markov-single}
            For two variables $X$ and $Y$, conditional causal effect $\interSp$ is \s-ID in DAG $\Grs$ if and only if 
            \begin{equation} \label{eq: singleton}
                 X \notin \Anc{S}{} \quad \text{ or } \quad (X \independent Y \vert S)_{\Grs_{\underline{X}}}.
            \end{equation}
        \end{theorem}
    
        Detailed proofs of our results appear in Appendices B and C.
        In the main text, we provide concise proof sketches to emphasize the key steps of our proofs.
    
        \begin{myproof}[Sketch of proof]
            \textit{Sufficiency.}
            Suppose Equation \eqref{eq: singleton} holds.
           Applying do-calculus rules allows us to show the following cases.
            \begin{itemize}
                \item If $(X \independent Y \vert S)_{\Grs_{\underline{X}}}$, then $\prs_{X}(Y) = \prs(Y \vert X)$.
                \item If $X \notin \Anc{S}{}$ and $Y \in \Pa{X}{}$, then $\prs_{X} (Y) = \prs (Y)$.
                \item If $X \notin \Anc{S}{}$ and $Y \notin \Pa{X}{}$, then
                \begin{equation*}
                     \prs_{X} (Y) = \mysum{\Pa{X}{}}{}\prs(Y \vert X, \Pa{X}{})\prs(\Pa{X}{}).
                \end{equation*}
            \end{itemize}
    
            \textit{Necessity.}
            For the necessary part, which is the challenging part of the proof, we need to show that when $X \in \Anc{S}{}$ and $(X \notindependent Y \vert S)_{\Grs_{\underline{X}}}$, then $\interSp$ is not \s-ID in $\Grs$.
            We first consider a special case where $Y \in \Anc{X}{}$ and prove the following in the appendix.
            
            \begin{claim}
                If $Y \in \Anc{X}{}$ and $X \in \Anc{S}{}$, then $\prs_X(Y)$ is not \s-ID in $\Grs$.
            \end{claim}
    
            Accordingly, to complete the proof, suppose $Y \notin \Anc{X}{}$.
    
            \begin{claim} \label{lem: subgraph}
                If $\interSp$ is not \s-ID in a subgraph of $\Grs$, then $\interSp$ is not \s-ID in $\Grs$.
            \end{claim}
    
            To prove the theorem using Claim \ref{lem: subgraph}, we first introduce a subgraph of $\Grs$ and then show that $\interSp$ is not \s-ID in that subgraph.
            
            \begin{claim} \label{claim: min collider}
               There exists a path between $X$ and $Y$ in $\Grs_{\underline{X}}$, which is not blocked by $S$, and it contains at most one collider.
            \end{claim}
    
            Denote by $\PR$, a path between $X$ and $Y$ in $\Grs_{\underline{X}}$ with the minimum number of colliders such that $S$ does not block $\PR$.
            Due to Claim \ref{claim: min collider}, path $\PR$ exists and has at most one collider.
    
            Let $\Gr'$ be a minimal (in terms of edges) subgraph of $\Grs$ such that
            \begin{enumerate*}[label=(\roman*)]
                \item $\Gr'$ contains $\PR$,
                \item $X \in \Anc{S}{\Gr'}$, and
                \item if $\mathcal{P}$ has exactly one collider, then the collider is an ancestor of $S$ in $\Gr'$.
            \end{enumerate*}
            Note that if $\PR$ has a collider, then it is an ancestor of $S$ in $\Grs$ since $S$ does not block $\PR$.
            Thus, graph $\Gr'$ with these properties exists.
    
            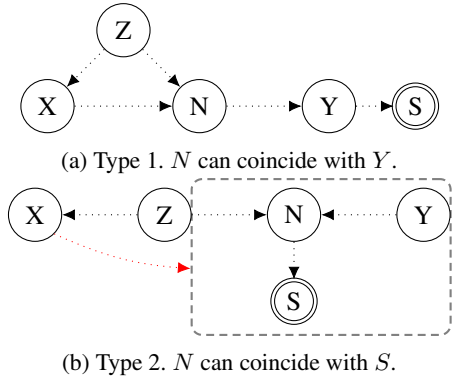
\begin{figure}
                \centering
                \begin{subfigure}[b]{0.4\textwidth}
                    \centering
                    \begin{tikzpicture}
                        \tikzset{edge/.style = {->,> = latex',-{Latex[width=1.5mm]}}}
                        \node [block, label=center:Z](Z) {};
                        \node [block, below right = 0.5 and 0.5 of Z, label=center:N](N) {};
                        \node [block, right = 1 of N, label=center:Y](Y) {};
                        \node [block, below left = 0.5 and 0.5 of Z, label=center:X](X) {};
                        \node [block-s, right = 0.5 of Y, label=center:S](S) {};
                        \draw[edge, dotted] (Z) to (N);
                        \draw[edge, dotted] (N) to (Y);
                        \draw[edge, dotted] (X) to (N);
                        \draw[edge, dotted] (Y) to (S);
                        \draw[edge, dotted] (Z) to (X);
                    \end{tikzpicture}
                    \caption{Type 1. $N$ can coincide with $Y$.}
                    \label{fig:type1}
                \end{subfigure}
                
                \begin{subfigure}[b]{0.4\textwidth}
                    \centering
                    \begin{tikzpicture}[
                        node distance = 1ex and 0em,
                        outer/.style={draw=gray, thick, rounded corners,
                                      densely dashed, fill=white!5,
                                      inner xsep=0ex, xshift=0ex, inner ysep=2ex, yshift=1ex,
                                      fit=#1}
                              ]
                        \tikzset{edge/.style = {->,> = latex',-{Latex[width=1.5mm]}}}
                        \node [block, label=center:Z](Z) {};
                        \node [block, right = 1 of Z, label=center:N](N) {};
                        \node [block, right = 1 of N, label=center:Y](Y) {}; 
                        \node [block, left = 1 of Z, label=center:X](X) {};
                        \node [block-s, below = 0.5 of N, label=center:S](S) {};
                        \draw[edge, dotted] (Z) to (N);
                        \draw[edge, dotted] (Y) to (N);
                        \draw[edge, dotted] (N) to (S);
                        \draw[edge, dotted] (Z) to (X);
                        \scoped[on background layer]
                        \node [outer=(Z.east)(S) (Y.east),
                             label={[anchor=north]:}](M) {};
                        \draw[edge, dotted, red] (X.south east) to [bend right=10] (M);
                
                    \end{tikzpicture}  
                    \caption{Type 2. $N$ can coincide with $S$.}
                    \label{fig:type2}
                \end{subfigure}
                \caption{Two types of DAGs used in the proof of Theorem \ref{th:markov-single}. The dotted edges indicate the presence of a directed path.} 
                \label{fig:markov_cases}
            \end{figure}

            Figure \ref{fig:markov_cases} illustrates two types of DAGs, where the dotted edges indicate the presence of a directed path, and the directed paths do not share any edges.
            Variable $N$ can coincide with $Y$ in Figure \ref{fig:type1} and with $S$ in Figure \ref{fig:type2}.
            Furthermore, in Figure \ref{fig:type2}, the directed path in red is towards a variable inside the box, i.e., the variables in the directed paths from $Z$ to $N$ (except $Z$ itself), from $N$ to $S$, and from $Y$ to $N$.
            
            In the appendix, we introduce a series of transformations to simplify $\Gr'$ and convert it to one of the two forms depicted in Figure \ref{fig:markov_cases}.
            Denote by $\Gr''$, the DAG obtained by this conversion.
            This conversion ensures that if $\interSp$ is not \s-ID in $\Gr''$, then it is not \s-ID in $\Gr'$.
            Therefore, it suffices to show that $\interSp$ is not \s-ID in $\Gr''$.
            To this end, in the appendix, we introduce two SEMs with causal graph $\Gr''$, denoted by $\M_1$ and $\M_2$, and show that $\pro^{\M_1}_{X}(Y \vert S = 1) \neq \pro^{\M_2}_{X} (Y \vert S = 1)$, while $\pro^{\M_1}(\Vb \vert S = 1) = \pro^{\M_2}(\Vb \vert S = 1)>0$.
            This proves that $\interSp$ is not \s-ID in $\Gr''$ and completes the proof.
        \end{myproof}
    
        \begin{figure}[t]
            \centering
            \begin{subfigure}{0.45 \textwidth}
                \centering
                \resizebox{0.45\textwidth}{!}{     
                    \begin{tikzpicture}
                                \tikzset{edge/.style = {->,> = latex',-{Latex[width=1.5mm]}}}
                                \node [block, label=center:X](X) {};
                               
                                \node [block-s, below right= 1 and 1 of X, label=center:S](S) {};
                                 \node [block, above right= 1 and 1 of S, label=center:Y](Y) {};
                                \draw[edge] (X) to (S);
                                \draw[edge] (Y) to (S);
                    \end{tikzpicture}
                } \hfill
                \resizebox{0.45\textwidth}{!}{
                    \begin{tikzpicture}
                        \tikzset{edge/.style = {->,> = latex',-{Latex[width=1.5mm]}}}
                        \node [block, label=center:X](X) {};  
                        \node [block-s, below right= 1 and 1 of X, label=center:S](S) {};
                         \node [block, above right= 1 and 1 of S, label=center:Y](Y) {};
                        \draw[edge] (X) to (Y);
                        \draw[edge] (Y) to (S);
                    \end{tikzpicture}
                }
                \caption{$\interSp = \prs(Y \vert X)$}
                \label{fig:singleton-main-a}
            \end{subfigure}
            
            \begin{subfigure}{0.45\textwidth}
                \centering
                \resizebox{.6\textwidth}{!}{              
                    \begin{tikzpicture}
                                \tikzset{edge/.style = {->,> = latex',-{Latex[width=1.5mm]}}}
                                \node [block-s, label=center:S](S) {};
                                \node [block, above right= 0.5 and 0.5 of S, label=center:W](W) {};
                                \node [block, above left= 0.5 and  0.5 of S, label=center:Z](Z) {};
                                \node [block, below left= 0.3 and 2 of S, label=center:X](X) {};
                                \node [block, below right= 0.3 and 2 of S, label=center:Y](Y) {};
                                \draw[edge] (Z) to (X);
                                \draw[edge] (X) to (Y);
                                \draw[edge] (Z) to (S);
                                \draw[edge] (W) to (S);
                                \draw[edge] (W) to (Y);
                    \end{tikzpicture}
                }
                \caption{$\interSp = \mysum{Z}{}\prs(Y\vert X, Z)\prs(Z)$}
                \label{fig:singleton-main-b}
            \end{subfigure}
            \caption{Three DAGs in which $\interSp$ is \s-ID.}
            \label{fig:singleton_main}
        \end{figure}
    
        Figure \ref{fig:singleton_main} depicts three example graphs in which $\interSbp$ is \s-ID.
        In both DAGs in Figure \ref{fig:singleton-main-a}, $(X \independent Y \vert S)_{\Grs_{\underline{X}}}$.
        As mentioned in the sketch of proof of Theorem \ref{th:markov-single}, this implies that $\prs_{X}(Y) = \prs(Y \vert X)$.
        We note that in the left graph, $X$ does not have any causal effect on $Y$ in the population (i.e., $\pro_X(Y) = \pro(Y)$) since $Y$ is not a descendent of $X$.
        However, $X$ has causal effect on $Y$ in the sub-population (i.e., $\interSp \neq \prs(Y)$) due to the dependency of $X$ and $Y$ in $\prs$.
        In Figure \ref{fig:singleton-main-b}, since $X \notin \Anc{S}{} = \{Z,W,S\}$ and $Y \notin \Pa{X}{} = \{Z\}$, we have $\interSp = \sum_{Z}\prs(Y\vert Z, X)\prs(Z)$.
        Note that while $\pro_{X}(Y) = \pro(Y\vert X)$, $\prs_{X}(Y) \neq \prs(Y\vert X)$.
        \begin{remark}
            These examples show that ignoring $S$ and assuming that our available samples are generated from $\pro$ (as opposed to $\prs$) might lead to erroneous inferences.
        \end{remark}
    
        \begin{figure}[t]
            \centering
            \begin{subfigure}[b]{0.18\textwidth}
                \centering
                \begin{tikzpicture}
                    \tikzset{edge/.style = {->,> = latex',-{Latex[width=1.5mm]}}}
                    \node [block, label=center:X](X) {};
                    \node [block, above right= 0.5 and 0.5 of X, label=center:Z](Z) {};
                    \node [block, below right= 0.5 and 0.5 of Z, label=center:Y](Y) {};
                    \node [block-s, below right= 0.5 and 0.5 of X, label=center:S](S) {};
                    \draw[edge] (Z) to (X);
                    \draw[edge] (Z) to (Y);
                    \draw[edge] (X) to (Y);
                    \draw[edge] (Y) to (S);
                \end{tikzpicture}
            \end{subfigure}
            \hfill
            \begin{subfigure}[b]{0.28\textwidth}
                \centering
                \begin{tikzpicture}
                    \tikzset{edge/.style = {->,> = latex',-{Latex[width=1.5mm]}}}
                    \node [block, label=center:X](X) {};
                    \node [block, above right=  1 and 0.3 of X, label=center:Z](Z) {};
                    \node [block, right= 2.2 of X, label=center:Y](Y) {};
                    \node [block, above left = 1 and 0.3 of Y, label=center:W](W) {};
                    \node [block-s, right= 0.5 of W, label=center:S](S) {};
                    \draw[edge] (Z) to (X);
                    \draw[edge] (Z) to (W);
                    \draw[edge] (X) to (Y);   
                    \draw[edge] (Y) to (W);    
                    \draw[edge] (W) to (S);
                \end{tikzpicture}    
            \end{subfigure}
            \caption{Two DAGs in which $\interSp$ is not \s-ID.}
            \label{fig: singleton-not s-ID}
        \end{figure}
    
        Figure \ref{fig: singleton-not s-ID} depicts two DAGs in which $X \in \Anc{S}{}$ and $(X \notindependent Y \vert S)_{\Grs_{\underline{X}}}$ (in the left graph $X \gets Z \to Y$ and in the right graph $X \gets Z \to W \gets Y$ is an active path in $\Grs_{\underline{X}}$).
        Hence, Equation \eqref{eq: singleton} does not hold, and Theorem \ref{th:markov-single} implies that $\interSp$ is not \s-ID.
    
    \subsection{Conditions for s-Identifiability: Multivariate Case}
        We present a necessary and sufficient condition for $\interSbp$ to be \s-ID in DAG $\Grs$ in the multivariate case.
        To do so, we decompose $\Xb$ into two parts: ancestors and non-ancestors of $S$.
        The following proposition demonstrates that the conditional causal effect of the latter portion of $\Xb$ on any other subset is always \s-ID.
    
        \begin{proposition} \label{prp: X2}
            Suppose $\Grs$ is an augmented DAG over $\Vb \cup \{S\}$, and let $\Xb \subsetneq \Vb$.
            For $\Xb_{2} \coloneqq \Xb \setminus \Anc{S}{}$, conditional causal effect $\prs_{\Xb_2} (\Vb \setminus \Xb_2)$ is \s-ID in $\Grs$ and can be computed from $\prs (\Vb)$ by
            \begin{equation} \label{eq: X2}
                \prs (\Anc{S}{} \setminus S) \prod_{W \in \Wb} \prs (W \vert \Pa{W}{}),
            \end{equation}
            where $\Wb =  \Vb \setminus (\Xb_2 \cup \Anc{S}{})$.
        \end{proposition}
        \begin{corollary}\label{coro}
            For $\Yb \subseteq \Vb \setminus \Xb_{2}$, conditional causal effect $\prs_{\Xb_2} (\Yb)$ is \s-ID in $\Grs$ since
            \begin{equation} \label{eq: X2 to Y}
                \prs_{\Xb_2} (\Yb) = \sum_{\Vb \setminus ( \Xb_2 \cup \Yb)} \prs_{\Xb_2} (\Vb \setminus \Xb_2).
            \end{equation}
        \end{corollary}
    
        So far, we have shown that $\prs_{\Xb_2} (\Vb \setminus \Xb_2)$ is always \s-ID in $\Grs$, where $\Xb_{2} = \Xb \setminus \Anc{S}{}$.
        The following theorem provides a necessary and sufficient condition for $\interSbp$ to be \s-ID in $\Grs$.
        When this condition holds, the following theorem presents a formula to compute $\interSbp$ in terms of $\prs_{\Xb_2} (\Vb \setminus \Xb_2)$, which is always \s-ID as established in Corollary \ref{coro}.
    
        \begin{theorem}\label{th:markov-general}
            For disjoint subsets $\Xb$ and $\Yb$ of $\Vb$, let $\Xb_{1} \coloneqq \Xb \cap \Anc{S}{}$ and $\Xb_{2} \coloneqq \Xb \setminus \Anc{S}{}$.
            
            \begin{itemize}
                \item If $\Xb_1 = \varnothing$:  Conditional causal effect $\prs_{\Xb} (\Yb)$ is \s-ID and can be computed from Equation \eqref{eq: X2 to Y}.
                \item If $\Xb_1 \neq \varnothing$: Conditional causal effect $\prs_{\Xb} (\Yb)$ is \s-ID if and only if
                    \begin{equation} \label{eq: general case cond}
                        (\Xb_{1} \independent \Yb \vert \Xb_{2}, S)_{\Grs_{\underline{\Xb_{1}}\overline{\Xb_{2}}}}.
                    \end{equation}
                    Moreover, when \eqref{eq: general case cond} holds, we have 
                    \begin{equation}\label{eq:general_formula}
                        \prs_{\Xb} (\Yb) = \prs (\Xb_1)^{-1} \sum_{\Vb \setminus (\Xb \cup \Yb)} \prs_{\Xb_2} (\Vb \setminus \Xb_2),
                    \end{equation}
                    where $\prs_{\Xb_2} (\Vb \setminus \Xb_2)$ can be computed from $\prs(\Vb)$ using Equation \eqref{eq: X2}.
            \end{itemize}
        \end{theorem}
        \begin{corollary}
            Conditional causal effect $\interSbp$ is \s-ID in $\Grs$ if and only if 
            \begin{equation} \label{eq: multivariate}
                 \Xb_1 = \varnothing \quad \text{ or } \quad (\Xb_{1} \independent \Yb \vert \Xb_{2}, S)_{\Grs_{\underline{\Xb_{1}}\overline{\Xb_{2}}}}.
            \end{equation}
            Furthermore, if $\Xb= \{X\}$ is singleton, then $\Xb_1 = \varnothing$, which is equivalent to $X \notin \Anc{S}{}$ and Theorem \ref{th:markov-general} reduces to Theorem \ref{th:markov-single} for the singleton case.
        \end{corollary}

        \begin{myproof}[Sketch of proof]
            The first part of the theorem (if $\Xb_1 = \varnothing$) is a direct consequence of Proposition \ref{prp: X2}.
            To show the second part, we assume $\Xb_1 \neq \varnothing$.
            
            \textit{Sufficiency.}
            Suppose Equation \eqref{eq: general case cond} holds.
            We need to show that Equation \eqref{eq:general_formula} holds.
            By applying Rules 2 and 3 of do-calculus, it can be shown that
            \begin{equation*}
                 \prs_{\Xb}(\Yb) = \prs_{\Xb_{2}}(\Yb \vert \Xb_{1}) = \frac{\prs_{\Xb_{2}}(\Xb_{1}, \Yb)}{\prs_{\Xb_{2}}(\Xb_{1})} = \frac{\prs_{\Xb_{2}}(\Xb_{1}, \Yb)}{\prs(\Xb_{1})}.
            \end{equation*}
            Moreover, Corollary \ref{coro} for $\Xb_{1} \cup \Yb$ implies that 
            \begin{equation*}
                \prs_{\Xb_{2}}(\Xb_{1}, \Yb) = \sum_{\Vb \setminus (\Xb \cup \Yb)} \prs_{\Xb_2} (\Vb \setminus \Xb_2).
            \end{equation*}
            Equation \eqref{eq:general_formula} can be obtained by merging the above equations.
        
            \textit{Necessity.}
            Suppose $(\Xb_{1} \notindependent \Yb \vert \Xb_{2}, S)_{\Grs_{\underline{\Xb_{1}}\overline{\Xb_{2}}}}$.
            We need to show that $\prs_{\Xb}(\Yb)$ is not \s-ID in $\Grs$.   
            \begin{claim}
                \label{lemma:subgraph-non-id}
                There exists $X^* \in \Xb_1$, $Y^* \in \Yb$, and a subgraph $\Gr^*$ of $\Grs$ such that 
                \begin{itemize}
                    \item $\Xb \cap \Anc{S}{\Gr^*_{\overline{X^*}}} = \{X^*\}$,
                    \item $(X^* \notindependent Y^* \vert S)_{\Gr^*_{\underline{X^*}}}$, and
                    \item $(\Xb\setminus \{X^*\} \independent Y^* \vert X^*, S)_{\Gr^*_{\overline{\Xb}}}$.
                \end{itemize}
            \end{claim}
            The first property implies that $X^* \in \Anc{S}{\Gr^*}$.
            Hence, Equation \eqref{eq: singleton} holds for $X^*$ and $Y^*$ in $\Gr^*$ and Theorem \ref{th:markov-single} implies that $\prs_{X^*}(Y^*)$ is not \s-ID in $\Gr^*$.
            To conclude the proof, similar to the sketch of proof of Theorem \ref{th:markov-single}, it suffices to show that $\interSbp$ is not \s-ID in $\Gr^*$.
            
            For any SEM with causal graph $\Gr^*$, due to the first and third properties in Claim \ref{lemma:subgraph-non-id}, Rule 3 of do-calculus implies that $\pro^{s}_{X^*}(Y^*) = \pro^{s}_{\Xb} (Y^*)$.
            Therefore, in $\Gr^*$, the s-identifiability of $\pro^{s}_{X^*}(Y^*)$ is equivalent to s-identifiability of $\pro^{s}_{\Xb}(Y^*)$, thus $\pro^{s}_{\Xb}(Y^*)$ is not \s-ID in $\Gr^*$.
            This shows that $\pro^{s}_{\Xb}(\Yb)$ is also not \s-ID in $\Gr^*$ since $Y^* \in \Yb$, which concludes our proof.
        \end{myproof}
    
        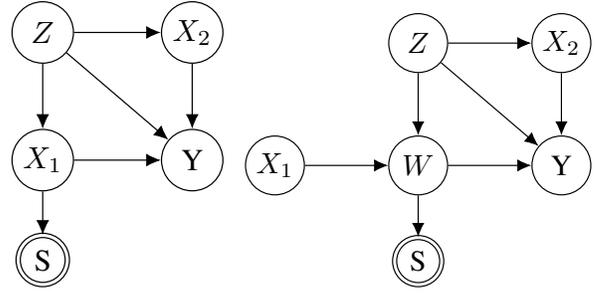
\begin{figure}
            \centering
            \resizebox{0.17\textwidth}{!}{      
                \begin{tikzpicture}
                            \tikzset{edge/.style = {->,> = latex',-{Latex[width=1.5mm]}}}
                            
                            \node [block, label=center:$X_1$](X1) {};
                            \node [block-s, below = 0.5 of X1, label=center:S](S) {}; 
                            \node [block, right= 1 of X1, label=center:Y](Y) {};
                            \node [block, above = 0.75 of X1, label=center:$Z$](Z) {};
                            \node [block, above = 0.75 of Y, label=center:$X_2$](X2) {};
                            
                            \draw[edge] (Z) to (Y);
                            \draw[edge] (Z) to (X1);
                            \draw[edge] (Z) to (X2);
                            \draw[edge] (X1) to (S);
                            \draw[edge] (X1) to (Y);
                            \draw[edge] (X2) to (Y);
                \end{tikzpicture}
            }
            \resizebox{0.27\textwidth}{!}{
                 \begin{tikzpicture}
                    \tikzset{edge/.style = {->,> = latex',-{Latex[width=1.5mm]}}}
                    
                    \node [block, label=center:$W$](W) {};
                    \node [block, left = 1 of W, label=center:$X_1$](X1) {};
                    \node [block-s, below = 0.5 of W, label=center:S](S) {}; 
                    \node [block, right= 1 of W, label=center:Y](Y) {};
                    \node [block, above = 0.75 of W, label=center:$Z$](Z) {};
                    \node [block, above = 0.75 of Y, label=center:$X_2$](X2) {};
                    
                    \draw[edge] (X1) to (W);
                    \draw[edge] (Z) to (Y);
                    \draw[edge] (Z) to (W);
                    \draw[edge] (Z) to (X2);
                    \draw[edge] (W) to (S);
                    \draw[edge] (W) to (Y);
                    \draw[edge] (X2) to (Y);
                 \end{tikzpicture}
             }     
             \caption{An example for the multivariate case where conditional causal effect $\prs_{\{X_1, X_2\}}(Y)$ is not \s-ID in the left graph while it is \s-ID in the right graph and is equal to $\mysum{Z, W}{} \prs(Z, W \vert X_1) \prs(Y \vert X_2, Z, W)$.}
            \label{fig:general_ex}
        \end{figure}
    
        Consider the two DAGs in Figure \ref{fig:general_ex}, where we are interested in computing $\interSbp$ for $\Xb = \{X_1, X_2\}$ and $\Yb = \{Y\}$.
        Accordingly, we have $\Xb_1 = \Xb\cap\Anc{S}{} = \{X_1\}$ and $\Xb_2 =\Xb \setminus \Anc{S}{}= \{X_2\}$ for both DAGs.
        In the DAG on the left, $(X_1 \notindependent Y \vert X_2, S)_{\Grs_{\underline{X_1}\overline{X_{2}}}}$ since $X_1 \leftarrow Z \rightarrow Y$ is an active path.
        Hence, Theorem \ref{th:markov-general} implies that $\prs_{\{X_1, X_2\}}(Y)$ is not \s-ID.
        On the other hand, $\prs_{\{X_1, X_2\}}(Y)$ is \s-ID in the  DAG on the right since $(X_1 \independent Y \vert X_2, S)_{\Grs_{\underline{X_1}\overline{X_{2}}}}$.
        Moreover, $\prs_{X_2}(Y) = \prs(X_1, Z, W) \prs(Y \vert X_2, Z, W)$ due to Proposition \ref{prp: X2}, thus, $\prs_{\{X_1, X_2\}}(Y) = \mysum{Z, W}{} \prs(Z, W \vert X_1) \prs(Y \vert X_2, Z, W)$.
        Note that $ \prs(X_1)^{-1}\prs(X_1, Z, W) = \prs(Z, W \vert X_1)$.

    \subsection{A Sound And Complete Algorithm For \s-ID}
        Equipped by Proposition \ref{prp: X2} and Theorem \ref{th:markov-general}, we present Algorithm \ref{algo: SID} for the \s-ID problem.\footnote{Our implementation is at \url{https://github.com/amabouei/s-ID}.}
        The inputs are the set of intervened variables $\Xb$, the set of outcome variables $\Yb$, augmented causal DAG $\Grs$, and the observational distribution of the target sub-population $\prs(\Vb)$.
        The algorithm returns a formula for conditional causal effect $\interSbp$ based on $\prs(\Vb)$ when it is \s-ID in $\Grs$.
        Otherwise, it returns \textsc{Fail} which indicates that $\interSbp$ is not \s-ID in $\Grs$.
    
        \begin{algorithm}[t!]
            \caption{A sound and complete algorithm for \s-ID}
            \label{algo: SID}
            \begin{algorithmic}[1]
                \STATE \textbf{Input:} $\Xb, \Yb, \Grs, \prs(\Vb)$
                \STATE \textbf{Output:} A formula for $\prs_{\Xb}(\Yb)$ based on $\prs(\Vb)$ if it is \s-ID, otherwise, \textsc{Fail}
                \STATE $\Xb_1 \gets \Xb \cap \Anc{S}{}$
                \STATE $\Xb_2 \gets \Xb \setminus \Anc{S}{}$
                \STATE $\Wb \gets  \Vb \setminus (\Xb_2 \cup \Anc{S}{})$  
                \IF{$\Xb_1 = \varnothing$}    
                    \STATE \textbf{Return} $\mysum{\Vb \setminus (\Xb \cup \Yb)}{}\!\!\!\!\!\! \prs(\Anc{S}{}\setminus S)\!\! \myprod{W \in \Wb}{}\!\!\!\! \prs(W \vert \Pa{W}{})$
                \ELSIF{$(\Xb_{1} \independent \Yb \vert \Xb_{2}, S)_{\Grs_{\underline{\Xb_{1}}\overline{\Xb_{2}}}}$}
                    \STATE \textbf{Return} $\mysum{\Vb \setminus (\Xb \cup \Yb)}{}\!\!\!\!\!\! \frac{\prs(\Anc{S}{}\setminus S)}{\prs(\Xb_1)} \myprod{W \in \Wb}{}\prs(W \vert \Pa{W}{})$
                \ELSE
                    \STATE \textbf{Return} \textsc{Fail}
                \ENDIF
            \end{algorithmic}
        \end{algorithm}
    
        The algorithm starts by decomposing $\Xb$ to ancestors ($\Xb_1$) and non-ancestors ($\Xb_2$) of $S$.
        If $\Xb_1 = \varnothing$, due to the first part of Theorem \ref{th:markov-general}, $\interSbp$ is \s-ID and the algorithm returns Equation \eqref{eq: X2 to Y} by replacing $\prs_{\Xb_2} (\Vb \setminus \Xb_2)$ with Equation \eqref{eq: X2}.
        Otherwise (i.e., when $\Xb_1 \neq \varnothing$), the algorithm checks the s-identifiability condition of Equation \eqref{eq: general case cond} in line 8.
        If the condition holds, it returns Equation \eqref{eq:general_formula}, again, by replacing $\prs_{\Xb_2} (\Vb \setminus \Xb_2)$ with Equation \eqref{eq: X2}.
        If the condition does not hold, Theorem \ref{th:markov-general} implies that $\interSbp$ is not \s-ID, and the algorithm returns \textsc{Fail}.
    
        \begin{corollary}
            Algorithm \ref{algo: SID} is sound and complete for the \s-ID problem.
            That is, when $\interSbp$ is \s-ID in $\Grs$, it returns a sound formula for it based on $\prs(\Vb)$ (soundness), and otherwise, it returns \textsc{Fail} (completeness).
        \end{corollary}

        We can use efficient methods such as the one presented \citet{darwiche2009modeling} to verify the d-separation in Line 8 of Algorithm \ref{algo: SID}.
        Accordingly, the time complexity of the algorithm is $\mathcal{O}(n+m)$, where $n$ and $m$ represent the number of nodes and edges in the graph, respectively.

\section{Related Work} \label{sec: related work}
    In this section, we review related problems in the causal inference literature.
    A summary of the settings for these problems can be found in Table \ref{table:1}.
    
    \subsection{Causal Inference in Population}
        The goal of causal inference in the entire population is to compute a causal effect $\pro_{\Xb}(\Yb)$.
        The seminal ID problem \cite{pearl1995causal}, proposed by Judea Pearl, is concerned with calculating $\pro_{\Xb}(\Yb)$ based on observational distribution $\pro(\Vb)$ when the causal graph is known.
        Pearl proposed three fundamental rules known as do-calculus, which, along with probabilistic manipulations, can be used to compute interventional distributions.
        Applying these rules, \citet{tian2003ID} proposed an algorithm for the ID problem, and later, \citet{shpitser2006identification} and \citet{huang2006identifiability} concurrently and with two different approaches showed that the proposed algorithm is sound and complete for the ID problem.
        The former introduced a graph structure called \emph{Hedge} and showed that the existence of a hedge is equivalent to the non-identifiability of an interventional distribution in the setting of the ID problem.
        The latter showed that the identifiability of a causal effect $\pro_{\Xb}(\Yb)$ is equivalent to the identifiability of $Q[\Zb]\coloneqq \pro_{\Vb \setminus \Zb}(\Zb)$ for $\Zb = \Anc{\Yb}{\Gr_{\underline{\Xb}}}$.
        They then showed that the recursive algorithm by \citet{tian2003ID} is sound and complete for the identifiability of $Q$ distributions.
        
        A more generalized formulation of the ID problem is known as gID or general identifiability \cite{Lee2019GeneralIW, kivva2022revisiting}.
        Similar to the ID problem, the goal in gID is to compute a causal effect $\pro_{\Xb}(\Yb)$ but from  $\{\pro_{\Zb_i}(\Vb \setminus \Zb_i)\}_{i = 0}^{m}$ for some subsets $\{\Zb_i\}_{i=0}^{m}$ of observed variables.
        Hence, ID is a special case of gID when $m=0$ and $\Zb_0 = \varnothing$.
        \citet{kivva2022revisiting} extended the approach of \citet{huang2006identifiability} for the ID problem to gID and proposed a sound and complete algorithm for gID.
        
        Another problem in causal inference on population is the so-called S-Recoverability \cite{Bareinboim_Tian_Pearl_2014, Bareinboim_Tian_2015, Correa_Tian_Bareinboim_2019}.
        In contrast to ID and gID, the given distribution in S-Recoverability originates from a sub-population, yet the aim remains to calculate a causal effect for the entire population.
        The constraint of having data from merely a sub-population makes the inference task for the whole population particularly challenging.
        Consequently, it is plausible to anticipate that a majority of causal effects would be unidentifiable, a fact that inherently restricts the practical applicability of the S-Recoverability problem.
    
    \subsection{Causal Inference in a Sub-Population}
        \citet{shpitser2012identification} tackled the c-ID problem by proposing a sound and complete algorithm for computing a conditional causal effect $\pro_{\Xb}(\Yb \vert \Zb)$ from observational distribution $\pro(\Vb)$.
        They showed that $\Zb$ can be decomposed into two parts, namely $\Zb_1$ and $\Zb_2$, such that $\pro_{\Xb}(\Yb \vert \Zb)$ is c-ID if and only if $\pro_{\Xb_1, \Zb_1}(\Yb, \Zb_2)$ is ID.
        Hence, solving c-ID can be reduced to solving an ID problem.
        Similar to gID that generalizes the ID problem, \citet{correa2021nested} and \citet{kivva2023identifiability} generalized c-ID to c-gID.
        The objective in c-gID is again the computation of a conditional causal effect $\pro_{\Xb}(\Yb \vert \Zb)$, but from a set of interventional distributions of form $\{\pro_{\Zb_i}(\Vb \setminus \Zb_i)\}_{i = 0}^{m}$ instead of merely the observational distribution.
    
        Both the c-ID and c-gID settings operate on the premise that the given distributions originate from the entire population.
        Thus, to make an inference for a target sub-population, they require samples from the whole population (observational distribution in the case of c-ID and interventional distributions for c-gID).
        By contrast, the \s-ID problem uses the observational distribution merely from the target sub-population.
    
    \subsection{Causal Graph Variations}
        In all the aforementioned causal inference problems, the causal graph is assumed to be given.
        Although many causal discovery algorithms, such as the ones proposed by \citet{colombo2012learning, claassen2013learning, bernstein2020ordering, akbari2021recursive, huang2022latent, mokhtarian2021recursive, mokhtarian2023novel}, aim to learn the causal graph using observational distribution, the causal graph is only identifiable up to the so-called Markov equivalence class \cite{spirtes2000causation, pearl2009causality}.
        Addressing this gap, \citet{pmlr-v97-jaber19a} and \citet{NEURIPS2022_17a9ab41} provided algorithms for the ID and c-ID problems, respectively, where instead of the causal graph, a partial ancestral graph (PAG) that represents the equivalence class of the causal graph is known.
        \citet{akbari2023causal} consider the ID problem when the underlying graph is probabilistically defined.
        \citet{tikka2019identifying} and \citet{mokhtarian2022causal} consider a scenario for the ID problem where additional information about the causal graph is available in the form of context-specific independence (CSI) relations.
        They show that this side information renders more causal effects identifiable.

\section{Conclusion and Future Work}
    We introduced \s-ID, a practical scenario for causal inference in a sub-population.
    The \s-ID problem asks whether, given the causal graph, a causal effect in a sub-population can be identified from the observational distribution pertaining to the same sub-population.
    We provided a sound and complete algorithm for the \s-ID problem.
    While previous work, such as the c-ID and S-Recoverability problems, provide considerable insights, they cannot solve the \s-ID problem.
    Indeed through various examples, we demonstrated that ignoring the subtleties introduced by sub-populations in causal modeling can lead to erroneous inferences in the \s-ID problem.
    
    Our current framework assumes that all variables in the sub-population are observable.
    We acknowledge the potential practical situations where this may not be the case.
    Investigating the \s-ID problem in the presence of latent variables is an important future direction.
    Furthermore, to numerically estimate a causal effect, three key phases are involved: identification, estimation, and sensitivity analysis.
    This paper has addressed the identification problem, establishing a foundation for further research in the other two critical phases.

\section*{Acknowledgments}
    This research was in part supported by the Swiss National Science Foundation under NCCR Automation, grant agreement 51NF40\_180545 and Swiss SNF project 200021\_204355 /1.

\bibliography{aaai24}

\ifarxiv
    \appendix
    \clearpage
    \onecolumn
    \include{appendix}

\fi

\end{document}

%% file: commands.tex
\newcommand{\N}{\mathcal{N}}

\newcommand{\pro}{P}

\newcommand{\PR}{\mathcal{P}}

\newcommand{\Gr}{\mathcal{G}}
\newcommand{\M}{\mathcal{M}}

\newcommand{\s}[0]{\textsc{s}}
\newcommand{\Grs}{\mathcal{G}^{\s}}
\newcommand{\prs}{P^{\s}}

\newcommand{\mysum}[2]{\sum\limits_{#1} ^{ #2} }
\newcommand{\myprod}[2]{\prod\limits_{#1}^{#2}}

\newcommand{\independent}{\perp\!\!\!\perp}
\newcommand{\notindependent}{\not\!\perp\!\!\!\perp}

\newcommand{\Pa}[2]{\textit{Pa}_{#2}(#1)}
\newcommand{\Ch}[2]{\textit{Ch}_{#2}(#1)}

\newcommand{\Anc}[2]{\textit{Anc}_{#2}(#1)}

\newcommand{\Vb}{\mathbf{V}}
\newcommand{\Xb}{\mathbf{X}}
\newcommand{\Yb}{\mathbf{Y}}

\newcommand{\Wb}{\mathbf{W}}

\newcommand{\Zb}{\mathbf{Z}}

\newcommand{\sump}[1]{\text{SPa}_{#1}}
\newcommand{\expo}[1]{\exp \left ( #1 \right)}

\newcommand{\interS}[1][Y]{\pro_{X}(#1 \vert S = 1)}
\newcommand{\inter}[0]{\pro_{\Xb}(\Yb)}
\newcommand{\interSb}[1][\Yb]{\pro_{\Xb}(#1 \vert S = 1)}

\newcommand{\interSp}[0]{\prs_{X}(Y)}
\newcommand{\interSbp}[1][\Yb]{\prs_{\Xb}(#1)}

\newtheorem{theorem}{Theorem}
\newtheorem{corollary}{Corollary}
\newtheorem{lemma}{Lemma}

\newtheorem{proposition}{Proposition}
\newtheorem{claim}{Claim}
\newtheorem{definition}{Definition}
\newtheorem{remark}{Remark}
\newtheorem*{lemma*}{Lemma}

\newenvironment{myproof}[1][\proofname]{%
  \begin{proof}[#1]$ $\nobreak\ignorespaces
}{%
  \end{proof}
}

\newcommand*\diff{\mathop{}\!\mathrm{d}}

%% file: appendix.tex
\begin{center}
    {\Large \textbf{Appendix}}
\end{center}

In Appendix A, we have included additional examples of the \s-ID problem.
Appendix B contains some preliminary lemmas which are used throughout our proofs.
The proofs for the main results - Theorems \ref{th:markov-single}, \ref{th:markov-general}, and Proposition \ref{prp: X2} - can be found in Appendix C.

\section{A \quad Additional Example in the Domain of Finance}
    In this section, we present an example of the \s-ID problem in the domain of finance.
    This example involves a simplified causal graph tailored to a financial context, which represents various economic and policy-related variables and their causal relationships \citep{engen2004federal, gamber2019effect, bollen2011modeling}.

    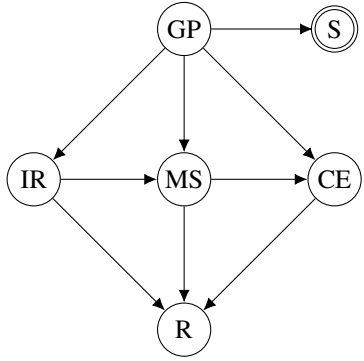
\begin{figure}[ht]
        \centering
        \tikzset{edge/.style = {->,> = latex',-{Latex[width=1.5mm]}}}
        \tikzstyle{input} = [coordinate]
        \tikzstyle{output} = [coordinate]
        \begin{tikzpicture}[->, auto, node distance=2cm,>=latex', every node/.style={inner sep=0.1cm}]
            \node [block, label=center:IR] (IR) {};
            \node [block, right of=IR, label=center:MS] (MS) {};
            \node [block, right of=MS, label=center:CE] (CE) {};
            \node [block, above of=MS, label=center:GP] (GP) {};
            \node [block, below of=MS, label=center:R] (R) {};
            \node [block-s, right of=GP, label=center:S] (S) {};
    
            \draw[edge] (IR) to (MS);
            \draw[edge] (MS) to (CE);
            \draw[edge] (GP) to (IR);
            \draw[edge] (GP) to (MS);
            \draw[edge] (GP) to (CE);
            \draw[edge] (IR) to (R);
            \draw[edge] (MS) to (R);
            \draw[edge] (CE) to (R);
            \draw[edge] (GP) to (S);
        \end{tikzpicture}
        \caption{Causal graph illustrating the relationships among various financial factors and their impact on Investment Returns.}
        \label{fig:revised_finance_causal_graph}
    \end{figure}
    
    \begin{itemize}
        \item \textbf{Government Policy (GP)}: Decisions and actions taken by the government that affect the economy. This can include fiscal policy, monetary policy, and regulatory changes.
        Government policy influences interest rates, market sentiment, and corporate earnings.

        \item \textbf{Interest Rates (IR)}: The cost of borrowing money, set by central banks, which is influenced by government policy. Interest rates affect market sentiment and directly impact investment returns.
        
        \item \textbf{Market Sentiment (MS)}: The overall mood or attitude of investors towards the financial markets. It is influenced by interest rates and government policy. Market sentiment affects corporate earnings (CE) and investment returns (R).
    
        \item \textbf{Corporate Earnings (CE)}: The profits made by companies, particularly those in a specific sector. Affected by market sentiment and government policy. Corporate earnings have a direct impact on investment returns.
    
        \item \textbf{Investment Returns (R)}: The gains or losses made on an investment. Influenced by market sentiment, interest rates, and corporate earnings.
    \end{itemize}

    Imagine a situation where the available data from the introduced variables is primarily derived from periods of expansive fiscal and monetary policies rather than a balanced representation of various policy regimes.
    This can occur if the data collection coincides with a specific economic era or a series of policy decisions that do not reflect the full spectrum of government interventions.
    To model this in our causal graph, we include an auxiliary variable $S$, added as a child of \textit{GP}.
    The condition $S=1$ corresponds to the sub-population of interest.

    In the setting of the \s-ID problem, we aim to compute the causal effect of Interest Rates on Investment Returns in our target sub-population.
    To this end, Theorem \ref{th:markov-single} implies that this causal effect is \s-ID and can be computed as follows.
    \begin{equation}
        \pro_{\textit{IR}}(R \vert S = 1) =  \mysum{\textit{GP}}{} \pro(R \vert \textit{IR}, \textit{GP}, S = 1) \pro(\textit{GP} \vert S = 1).
    \end{equation}

\section{B \quad Preliminary Lemmas}
    \begin{lemma}\label{lemma:remove}
        Suppose $\Grs$ is an augmented DAG over $\Vb \cup \{S\}$, and let $\Xb$,$\Yb$ be two disjoint subsets of $\Vb$.
        If $\prs_{\Xb}(\Yb)$ is not s-ID in $\Grs$, then $\prs_{\Xb}(\Yb)$ is not s-ID in any DAG obtained by adding a new edge to $\Grs$.
        Equivalently, if $\prs_{\Xb}(\Yb)$ is s-ID in $\Grs$, then $\prs_{\Xb}(\Yb)$ is s-ID in any DAG obtained by removing an edge from $\Grs$.
    \end{lemma}
    \begin{proof}
        Before presenting the proof, it is worth mentioning that a similar lemma with a similar proof also holds for the ID problem (Lemma 13 in \citet{tian2003ID}).
        
        Denote the graph obtained by adding a new edge $V_{i} \rightarrow V_{j}$ to $\Grs$ by $\Gr'$.
        Since $\prs_{\Xb}(\Yb)$ is not s-ID in $\Grs$, there exists two SEMs $\M_1$ and $\M_2$ with causal graph $\Grs$ such that 
        \begin{equation}\label{eq: proof edge remove}
            \begin{split} 
                & \pro^{\M_1} (\Vb \vert S = 1) = \pro^{\M_2} (\Vb \vert S = 1) > 0, \\
                & \pro^{\M_1}_{\Xb} (\Yb \vert S = 1) \neq \pro^{\M_2}_{\Xb} (\Yb \vert S = 1).
            \end{split}
        \end{equation}
        We now introduce two new SEMs $\M'_1$ and $\M'_2$ with causal graph $\Gr'$ as follows.
        \begin{equation*}
            \pro^{M'_{i}}(V \vert \Pa{V}{\Gr'}) \coloneqq \pro^{M_{i}}(V \vert \Pa{V}{\Gr}), \quad \forall V\cup \{S\}
            \in \Vb , i \in \{1,2\}
        \end{equation*}
        Note that $\Pa{V_j}{\Gr'} = \Pa{V}{\Gr} \cup \{V_i \}$ and for $V \in (\Vb \cup \{S\})\setminus \{V_j \}$, $\Pa{V}{\Gr'} = \Pa{V}{\Gr}$.
        It is straightforward to see that Equation \eqref{eq: proof edge remove} also holds for $\M'_1$ and $\M'_2$.
        Thus, $\interSb$ is not s-ID in $\Gr'$.
    \end{proof}

    \begin{lemma}\label{lemma:collider}
        Suppose $\Grs$ is an augmented DAG over $\Vb \cup \{S\}$ and let $X, Y$ be two distinct variables in $\Vb$.
        If $(X \notindependent Y \vert S)_{\Grs_{\underline{X}}}$, then there exists a path between $X$ and $Y$ in $\Grs_{\underline{X}}$ such that $S$ does not block it, and it contains at most one collider.
    \end{lemma}
    \begin{proof}
        Since $(X \notindependent Y \vert S)_{\Grs_{\underline{X}}}$, there exists at least one active path between $X$ and $Y$ in $\Grs_{\underline{X}}$ that $S$ does not block.
        Among all such paths, let $\mathcal{P} = (X, V_1, \dots, V_k, Y)$ denote an active path with the minimum number of colliders.
        It suffices to show that $\mathcal{P}$ has at most one collider.
        
        Assume, for the sake of contradiction, that $\mathcal{P}$ has at least two colliders.
        Denote by $V_l$ and $V_r$ the closest collider of $\mathcal{P}$ to $X$ and $Y$, respectively.
        Since $S$ does not block $\mathcal{P}$, all of the colliders on $\mathcal{P}$ must be an ancestor of $S$.
        Accordingly, denote by $\mathcal{P}_1$ and $\mathcal{P}_2$, directed paths from  $V_l$ to $S$ and from $V_r$ to $S$, respectively.
        Further, denote by $N$ the first intersection of $\mathcal{P}_1$ and $\mathcal{P}_2$ ($N$ can be $S$).
        Now, consider the path starting from $X$ to $V_l$ by the edges in $\mathcal{P}$, then from $V_l$ to $N$ by the edges in $\mathcal{P}_1$, then from $N$ to $V_r$ by the edges in $\mathcal{P}_2$, and then from $V_r$ to $Y$ by the edges in $\mathcal{P}$.
        In this path, $N$ is the only collider.
        Moreover, $S$ does not block it since $N$ is an ancestor of $S$.
        This is a contradiction and shows that $\mathcal{P}$ has at most one collider, which concludes the proof.
    \end{proof}

    \begin{lemma}\label{lemma:y_father}
        Suppose $\Grs$ is an augmented DAG over $\Vb \cup \{S\}$ and let $X, Y, Y'$ be three distinct variables in $\Vb$ such that $\Pa{Y}{} = \{Y'\}$ and $\Ch{Y}{\Grs} = \varnothing$.
        Further, denote by $\Gr'$ the DAG obtained by removing $Y$ from $\Grs$.
        If $\pro^{s}_{X}(Y')$ is not s-ID in $\Gr'$, then $\interSp$ is not s-ID in $\Grs$.
    \end{lemma}
    \begin{proof}
        Suppose $\pro^{s}_{X}(Y')$ is not s-ID in $\Gr'$.
        Thus, there exist two SEMs $\M_1$ and $\M_2$ with causal graph $\Gr'$ such that 
        \begin{equation}\label{eq: two models lemma 3}
            \pro^{\M_1}(\Vb \setminus {Y}\vert S = 1) = \pro^{\M_2}(\Vb \setminus {Y}\vert S = 1),
        \end{equation}
        and there exist $y^{*}$ and $x^{*}$ in the domains of $Y'$ and $X$, respectively, such that
        \begin{equation*}
                \pro^{\M_1}_{x^{*}} (y^{*} \vert S = 1) \neq  \pro^{\M_2}_{x^{*}} (y^{*} \vert S = 1).
        \end{equation*}
        Next, we introduce two SEMs $\M_1'$ and $\M_2'$ with causal graph $\Grs$.
        Let $i \in \{1,2\}$.
        In $\M'_{i}$, we define the structural equations of the variables in $\Vb \setminus \{Y\}$ similar to their equations in $\M_i$.
        Note that $Y$ does not have any children in $\Grs$, and $Y'$ is its only parent.
        We define $Y$ in both SEMs as a binary variable such that 
        \begin{equation*}
            \pro (Y=1 \vert Y'=y') = 
            1-\pro (Y=0 \vert Y'=y') =
            \begin{cases}
                 0.6, \quad \text{if } y' = y^{*}, \\
                 0.5, \quad \text{if } y' \neq y^{*}.
            \end{cases}
        \end{equation*}
        Accordingly, $\pro^{\M_i'} (\Vb \vert S = 1)$ can be computed as follows.
        \begin{equation} \label{eq: Mi' lemma 3}
            \pro^{\M'_i}(\Vb \vert S = 1) 
            = \pro^{\M_i}(\Vb \setminus \{Y\} \vert S = 1) \pro(Y \vert S = 1, \Vb \setminus \{Y\}) 
            = \pro^{\M_i}(\Vb \setminus \{Y\} \vert S = 1) \pro (Y \vert Y').
        \end{equation}
        The last equality holds because $(Y \independent (\Vb \cup \{S\}) \setminus \{Y,Y'\} \vert Y')_{\Grs}$.
        Hence, Equations \eqref{eq: two models lemma 3} and \eqref{eq: Mi' lemma 3} together show that 
        \begin{equation}
            \pro^{\M'_1} (\Vb \vert S = 1) = \pro^{\M'_2} (\Vb \vert S = 1).
        \end{equation}
        To complete the proof, it suffices to show that $\pro^{\M_1}_{x^{*}}(Y = 1 \vert S = 1) \neq  \pro^{\M_2}_{x^{*}}(Y = 1 \vert S = 1)$.
        In both SEMs, we have 
        \begin{equation*}
            \pro_{x^*}(Y \vert S = 1)= \sum_{Y'} \pro_{x^*}(Y, Y' \vert S=1) = \sum_{Y'} \pro_{x^*}(Y \vert Y', S=1)\pro_{x^*}(Y' \vert S=1).
        \end{equation*}
        Since $Y'$ is the only parent of $Y$ and $Y$ does not have any child, we have 
        \begin{equation*}
            \pro_{X}(Y \vert Y', S) = \pro(Y \vert Y').
        \end{equation*}
        Hence,
        \begin{align*}
            \pro_{x^*}(Y = 1 \vert S = 1) &
            = \sum_{Y'} \pro(Y=1 \vert Y') \pro_{x^*}(Y' \vert S=1) \\
            &= \pro(Y=1 \vert Y' = y^*) \pro_{x^*}(Y'=y^* \vert S=1) + \sum_{Y'\neq y^*} \pro(Y=1 \vert Y') \pro_{x^*}(Y' \vert S=1) \\
            &= 0.6 \pro_{x^*}(Y'=y^* \vert S=1) + 0.5 \sum_{Y'\neq y^*} \pro_{x^*}(Y' \vert S=1) \\
            &= 0.1 \pro_{x^*}(y^{*} \vert S=1) + 0.5.
        \end{align*}
        Since $\pro^{\M_1}_{x^{*}} (y^{*} \vert S = 1) \neq  \pro^{\M_2}_{x^{*}} (y^{*} \vert S = 1)$, this shows that $\pro^{\M'_1}_{x^{*}}(Y = 1 \vert S = 1) \neq \pro^{\M'_2}_{x^{*}}(Y = 1 \vert S = 1)$.
        This completes the proof.
    \end{proof}

    \begin{lemma}\label{lemma:normal-bernoulli}
        Let $R$ and $Y$ be two random variables such that
        \begin{equation}\label{eq: bernoulli 1}
            R \vert Y=y \sim \N(\mu_y, \sigma^2).
        \end{equation}
        Further, let $S$ be a binary variable such that 
        \begin{equation} \label{eq: bernoulli 2}
            S \vert Y = y, R=r
            \sim \textit{Bernoulli} (\frac{1}{\sqrt{2\pi}} \exp \left(- \frac{(r + b)^2}{2}\right),
        \end{equation}
        where $b$ is a real number.
        In this case, we have
        \begin{equation*}
            S \vert Y = y \sim \text{Bernoulli}\left( \frac{1}{\sqrt{2\pi(\sigma^2 + 1)}} \exp \left(- \frac{(\mu_y + b)^2}{2(\sigma^2 + 1)}\right) \right).
        \end{equation*}
    \end{lemma}

    \begin{proof}
        We have
        \begin{align*}
            \pro(S = 1 \vert Y = y) 
            &= \mysum{R}{} \pro(S = 1, R = r \vert Y = y)  \\
            &= \mysum{R}{} \pro(S = 1 \vert Y = y, R = r)  \pro(R = r \vert Y = y).
        \end{align*}
        By substituting $\pro(R = r \vert Y = y)$ and $\pro(S = 1 \vert Y = y, R = r)$ with Equations \eqref{eq: bernoulli 1} and \eqref{eq: bernoulli 2}, we have
        
        \begin{equation} \label{eq: lemma 4 - P(S=1)}
            \begin{split}
                \pro (S = 1 \vert Y = y)
                &= \frac{1}{\sqrt{2\pi}} \times \frac{1}{\sqrt{2\pi \sigma^2}} \int_{-\infty}^{\infty} \exp \left(- \frac{(r + b)^2}{2}\right) \exp \left(-\frac{(r - \mu_y)^2}{2\sigma^2}\right) \diff r \\
                &= \frac{1}{2 \pi \sigma} \exp \left(-\frac{(\mu_y + b)^2}{2(\sigma^2 + 1)}\right) \int_{-\infty}^{\infty} \exp \left(-\frac{(r - \frac{\mu_y - \sigma^2b}{\sigma^2 + 1})^2}{2\frac{\sigma^2}{\sigma^2 + 1}}\right) \diff r.
            \end{split}
        \end{equation}

        Now consider a Gaussian random variable with mean $\frac{\mu_y - \sigma^2b}{\sigma^2 + 1}$ and variance $\frac{\sigma^2}{\sigma^2 + 1}$.
        Since the integral is from $-\infty$ to $\infty$ of this variable's pdf is $1$, we have
        \begin{equation} \label{eq: lemma 4 - pdf}
             \int_{-\infty}^{\infty} \exp \left(-\frac{(r - \frac{\mu_y - \sigma^2b}{\sigma^2 + 1})^2}{2\frac{\sigma^2}{\sigma^2 + 1}}\right) \diff r = \sqrt{2\pi \frac{\sigma^2}{\sigma^2 + 1}}.
        \end{equation}
        Finally, by substituting Equation \eqref{eq: lemma 4 - pdf} in Equation \eqref{eq: lemma 4 - P(S=1)}, we have
        \begin{equation*}
            \pro(S = 1 \vert Y = y) = \frac{1}{\sqrt{2\pi(\sigma^2 + 1)}} \exp \left(- \frac{(\mu_y + b)^2}{2(\sigma^2 + 1)}\right).
        \end{equation*}
    \end{proof}

    \begin{lemma}\label{lemma:ancestor}
        Suppose $\Grs$ is an augmented DAG over $\Vb \cup \{S\}$ and let $X, Y$ be two distinct variables in $\Vb$.
        If $X \notin \Anc{S}{}$, then $\interSp$ is s-ID in $\Grs$.
    \end{lemma}
    \begin{proof}
        Since $X \notin \Anc{S}{}$, Rule 3 of do-calculus implies that $\pro_{X}(S = 1)= \pro(S = 1)$.
        Hence, we have
        \begin{equation*}
            \pro_{X}(Y \vert S = 1) 
            = \frac{\pro_{X}(Y, S = 1)}{\pro_{X}(S = 1)}
            = \frac{\pro_{X}(Y, S = 1)}{\pro(S = 1)}.
        \end{equation*}
        We divide the problem into two cases.
        \begin{itemize}
            \item \textbf{Case 1}: Suppose $Y \in \Pa{X}{}$.
            In this case, by applying Rule 3 of do-calculus, we have $\pro_{X}(Y, S = 1)= \pro(Y, S = 1)$.
            Hence,
            \begin{equation*}
                \prs_X(Y)
                = \pro_{X}(Y \vert S = 1)
                = \frac{\pro(Y,  S = 1)}{\pro(S = 1)} 
                = \pro(Y \vert S = 1)
                = \prs(Y).
            \end{equation*}
            
            \item \textbf{Case 2}: Suppose $Y \notin \Pa{X}{}$.
            \citet[Theorem 3.2.2]{pearl2000models} showed that we have
            \begin{align*}
                 \pro_{X}(Y, S = 1)
                 &= \sum_{\Pa{X}{}} \pro(Y,S = 1 \vert X, \Pa{X}{}) \pro (\Pa{X}{}) \\
                 &= \sum_{\Pa{X}{}} \pro(Y \vert X, \Pa{X}{}, S = 1) \pro(S = 1\vert X, \Pa{X}{}) \pro(\Pa{X}{}).
            \end{align*}
            Since $X \notin \Anc{S}{}$, we have $(X \independent S \vert \Pa{X}{})_{\Grs}$.
            Therefore, $\pro (S \vert X, \Pa{X}{}) = \pro (S \vert \Pa{X}{})$ and
            \begin{align*}
                 \pro_{X}(Y, S = 1)
                 &= \sum_{\Pa{X}{}} \pro(Y \vert X, \Pa{X}{}, S = 1) \pro(S = 1\vert \Pa{X}{}) \pro(\Pa{X}{})\\
                 &= \sum_{\Pa{X}{}} \pro(Y \vert X, \Pa{X}{}, S = 1) \pro(S = 1, \Pa{X}{}).
            \end{align*}
            Accordingly, we have
            \begin{align*}
                \prs_X(Y)
                &= \pro_{X}(Y \vert S = 1)
                = \frac{\pro_{X}(Y, S = 1)}{\pro(S = 1)}\\
                &= \pro(S = 1)^{-1} \sum_{\Pa{X}{}} \pro(Y \vert X, \Pa{X}{}, S = 1) \pro(S = 1, \Pa{X}{}) \\
                &= \sum_{\Pa{X}{}} \pro(Y \vert X, \Pa{X}{}, S = 1) \pro(\Pa{X}{} \vert S = 1) \\
                &= \sum_{\Pa{X}{}} \prs(Y \vert X, \Pa{X}{}) \prs(\Pa{X}{}). 
            \end{align*}
        \end{itemize}

        In both cases, we showed that $\prs_X(Y)$ is identifiable from $\prs(\Vb)$.
        Therefore, $\interSp$ is s-ID in $\Grs$.
    \end{proof}

    \begin{lemma} \label{lemma:Y ancestor of X}
        Suppose $\Grs$ is an augmented DAG over $\Vb \cup \{S\}$ and let $X, Y$ be two distinct variables in $\Vb$.
        If $Y \in \Anc{X}{}$ and $X \in \Anc{S}{}$, then $\prs_X(Y)$ is not s-ID in $\Grs$.
    \end{lemma}
    \begin{proof}
        Let $\PR_1$ and $\PR_2$ denote directed paths from $Y$ to $X$ and from $X$ to $S$, respectively.
        Note that these two paths do not intersect (except at $X$) since $\Gr$ is a DAG.
        Further, let $\Gr'$ be the subgraph of $\Grs$ consisting of only the edges of $\PR_1$ and $\PR_2$.
        For any SEM $\M$ with causal graph $\Gr'$, by applying rules 1 and 3 of do-calculus, we have
        \begin{equation} \label{eq: Lemma 6}
            \pro_{X}^{\M}(Y \vert S = 1) = \pro_{X}^{\M} (Y) = \pro^{\M}(Y).
        \end{equation}

        We use the following result to show that $\pro^{\M}(Y)$ is not identifiable from $\pro^{\M}(\Vb \vert S=1)$.

        \textbf{Theorem 1 in \citet{Bareinboim_Tian_Pearl_2014}}.
        \textit{
            Suppose $\mathcal{H}$ is a DAG over $\Vb$ and let $\Xb, \Yb$ be two distinct subsets in $\Vb$.
            Conditional distribution $\pro(\Yb \vert \Xb)$ is identifiable from $\pro(\Vb \vert S=1)$ if and only if $(\Yb \independent S \vert \Xb)_{\mathcal{H}^{s}}$.
        }

        Since $(Y \notindependent S)_{\Gr'}$, this theorem for $\mathcal{H} = \Gr'$, $\Xb = \varnothing$, and $\Yb = \{Y\}$ implies that $\pro^{\M}(Y)$ is not identifiable from $\pro^{\M}(\Vb \vert S=1)$.
        Therefore, $\interSp$ is not s-ID in $\Gr'$ due to Equation \eqref{eq: Lemma 6}.
        Finally, since $\Gr'$ is a subgraph of $\Grs$, Lemma \ref{lemma:remove} implies that $\prs_X(Y)$ is not s-ID in $\Grs$.
    \end{proof}

    \begin{lemma}\label{lemma:lem_point}
        Suppose $\Grs$ is an augmented DAG over $\Vb \cup \{S\}$ and let $X, Y$ be two distinct variables in $\Vb$.
        If $\interSp$ is s-ID in $\Grs$, then for any two SEMs with causal graph $\Grs$ such that $\pro^{\M_1}(\Vb \vert S = 1) = \pro^{\M_2}(\Vb \vert S = 1) > 0$, we have
        \begin{equation*}
            \frac{\pro^{\M_1}_{X}(Y = y_0, S = 1)}{\pro^{\M_2}_{X}(Y = y_0, S = 1)} 
            = \frac{\pro^{\M_1}_{X}(Y = y_1, S = 1)}{\pro^{\M_2}_{X}(Y = y_1 , S = 1)},
        \end{equation*}
        where $y_0$ and $y_1$ are arbitrary values in the domain of $Y$.
    \end{lemma}
 
    \begin{proof}
        Since $\interS$ is s-ID in $\Grs$, we have 
        \begin{equation*}
            \pro^{\M_1}_{X}(Y \vert S = 1) = \pro^{\M_2}_{X}(Y \vert S = 1).
        \end{equation*}
        Hence, 
        \begin{equation*}
            \frac{\pro^{\M_1}_{X}(Y, S = 1)}{\pro^{\M_1}_{X}(S = 1)}
            = \frac{\pro^{\M_2}_{X}(Y, S = 1)}{\pro^{\M_2}_{X}(S = 1)}.
        \end{equation*}
        Thus, for any $y$ in the domain of $Y$ we have
        \begin{equation*}
            \frac{\pro^{\M_1}_{X}(Y=y, S = 1)}{\pro^{\M_2}_{X}(Y=y, S = 1)}
            = \frac{\pro^{\M_1}_{X}(S = 1)}{\pro^{\M_2}_{X}(S = 1)}.
        \end{equation*}
        This shows that $\frac{\pro^{\M_1}_{X}(Y=y, S = 1)}{\pro^{\M_2}_{X}(Y=y, S = 1)}$ is independent of the value of $y$, which concludes the lemma.
    \end{proof}
   
    \begin{lemma} \label{lem: inequality}
        For any positive real number $a>0$, we have
        \begin{equation*}
            \frac{1+ e^{-1-a}}{1 + e^{1-a}} > e^{-1},
        \end{equation*}
        where $e$ denotes the Neper number.
    \end{lemma}
    \begin{proof}

        \begin{align*}
            \frac{1+ e^{-1-a}}{1 + e^{1-a}} > e^{-1}
            & \iff e+ e^{-a} > 1 + e^{1-a} \\
            & \iff e-1> e^{-a}(e-1)  \\
            & \iff e^a>1.
        \end{align*}
        The last inequality holds since $a>0$.
    \end{proof}

    \begin{lemma}[\citet{tian2003ID}]\label{lemma: margin}
        Suppose $\Gr$ is a DAG over $\Vb$ and let $\Vb' \subset \Vb$.
        If $\Anc{\Vb'}{\Gr} = \Vb'$, then  
        \begin{equation}
                Q[\Vb'] = \mysum{\Vb \setminus \Vb'}{} Q[\Vb] 
        \end{equation}
        Where $Q[\Zb] \coloneqq \pro_{\Vb \setminus \Zb}(\Zb)$, for any $\Zb \subseteq \Vb$.
    \end{lemma} 
    
    \begin{lemma} \label{lem: ancestral}
        Suppose $\Gr$ is a DAG over $\Vb$, and $\Vb' \subseteq \Vb$.
        If $\Anc{\Vb'}{\Gr} = \Vb'$, then 
        \begin{equation} \label{eq: proof ancestral}
            \pro(\Vb') = \prod_{V \in \Vb'} \pro(V \vert \Pa{V}{\Gr}).
        \end{equation}
    \end{lemma}

    \begin{proof}
          Lemma \ref{lemma: margin} implies that
          \begin{equation*}
              \pro_{\Vb \setminus \Vb'}(\Vb') 
              = Q[\Vb'] 
              = \mysum{\Vb \setminus \Vb'}{} Q[\Vb] 
              = \mysum{\Vb \setminus \Vb'}{} \pro(\Vb)
              = P(\Vb').
          \end{equation*}
          Using Markov truncated factorization, we have
          \begin{equation*}
              \pro_{\Vb \setminus \Vb'}(\Vb') = \prod_{V \in \Vb'} \pro(V \vert \Pa{V}{\Gr}).
          \end{equation*}
          Equation \eqref{eq: proof ancestral} can be obtained by merging the above equations.
    \end{proof}

    \begin{lemma} \label{lemma: Yb to Y}
        Suppose $\Grs$ is an augmented DAG over $\Vb \cup \{S\}$ and let $\Xb$ and $\Yb$ be two disjoint subsets of $\Vb$.
        For any $Y \in \Yb$, if $\prs_{\Xb}(Y)$ is not s-ID in $\Grs$, then $\prs_{\Xb}(\Yb)$ is not s-ID in $\Grs$.
    \end{lemma}
    \begin{proof}
        Since $\pro_{\Xb}(Y)$ is not s-ID in $\Grs$, there exist two SEMs $\M_1$ and $\M_2$ with causal graph $\Grs$ such that
        \begin{align*}            
            \pro^{\M_1}(\Vb \vert S = 1)
            &=  \pro^{\M_2} (\Vb \vert S = 1)>0, \\
            \pro_{\Xb}^{\M_1} (Y \vert S = 1) & \neq \pro_{\Xb}^{\M_2} (Y \vert S = 1).
        \end{align*}
        Assume, for the sake of contradiction, that $\pro_{\Xb}^{\M_1} (\Yb \vert S = 1) = \pro_{\Xb}^{\M_2} (\Yb \vert S = 1)$.
        Since
        \begin{equation*}
            \pro_{\Xb}(Y \vert S = 1) = \sum_{\Yb \setminus \{Y\}} \pro_{\Xb}(\Yb \vert S = 1), 
        \end{equation*}
        we have
        $\pro_{\Xb}^{\M_1} (Y \vert S = 1) = \pro_{\Xb}^{\M_2} (Y \vert S = 1)$, which is a contradiction.
        Therefore, we have
        \begin{equation*}
            \pro_{\Xb}^{\M_1} (\Yb \vert S = 1) \neq \pro_{\Xb}^{\M_2} (\Yb \vert S = 1),
        \end{equation*}
        thus, $\prs_{\Xb}(\Yb)$ is not s-ID in $\Grs$.
    \end{proof}

\section{C \quad Proofs of Main Results}
    In this section, we use the lemmas provided in Appendix B to prove Theorems 1,2 and Proposition 1.

\subsection{B.1 \quad Proof of Theorem 1}
    
    \textbf{Theorem 1.} \textit{For distinct variables $X$ and $Y$, conditional causal effect $\interSp$ is s-ID in DAG $\Grs$ if and only if}
        \begin{equation*}
             X \notin \Anc{S}{} \quad \text{ or } \quad (X \independent Y \vert S)_{\Grs_{\underline{X}}}.
        \end{equation*}

\subsubsection{Sufficiency.}
    If $X \notin \Anc{S}{}$, then Lemma \ref{lemma:ancestor} implies that $\interSp$ is s-ID in $\Grs$.
    Furthermore, if $(X \independent Y \vert S)_{\Grs_{\underline{X}}}$, then Rule 2 of do-calculus implies that 
    \begin{equation*}
        \prs_X(Y) = \interS = \pro (Y \vert X, S = 1) = \prs(Y \vert X).
    \end{equation*}
    Hence, $\interSp$ is s-ID in $\Grs$.

\subsubsection{Necessity.}
    We need to show that when $X \in \Anc{S}{}$ and $(X \notindependent Y \vert S)_{\Grs_{\underline{X}}}$, then $\interSp$ is not s-ID in $\Grs$.
    If $Y\in \Anc{X}{}$, then Lemma \ref{lemma:Y ancestor of X} shows that $\interSp$ is not s-ID in $\Grs$.
    Hence, for the rest of the proof, suppose $Y \notin \Anc{X}{}$.
    
    Lemma \ref{lemma:remove} implies that if $\interSp$ is not s-ID in a subgraph of $\Grs$, then $\interSp$ is not s-ID in $\Grs$.
    Hence, it suffices to introduce a subgraph of $\Grs$ and show that $\interSp$ is not s-ID in it.

    Denote by $\PR$, a path between $X$ and $Y$ in $\Grs_{\underline{X}}$ with the minimum number of colliders such that $S$ does not block $\PR$.
    Lemma \ref{lemma:collider} shows that $\PR$ exists and has at most one collider.

    \begin{figure}[ht]
        \centering
        \begin{subfigure}[b]{1\textwidth} 
            \centering
            \begin{tikzpicture}
                \tikzset{edge/.style = {->,> = latex',-{Latex[width=1.5mm]}}}
                \node [block, label=center:Z](Z) {};
                \node [block, left = 1 of Z, label=center:X](X) {};
                \node [block, right = 1 of Z, label=center:Y](Y) {};
                \draw[edge, dotted] (Z) to (Y);
                \draw[edge, dotted] (Z) to (X);   
            \end{tikzpicture}
            \caption{Path $\PR$ when it has no colliders.}
            \label{fig:non-collider_path}
        \end{subfigure}
        
        \begin{subfigure}[b]{1\textwidth}
            \centering
            \begin{tikzpicture}
                \tikzset{edge/.style = {->,> = latex',-{Latex[width=1.5mm]}}}
                \node [block, label=center:Z](Z) {};
                \node [block, left = 1 of Z, label=center:X](X) {};
                \node [block, right = 1 of Z, label=center:W](W) {};
                \node [block, right = 1 of W, label=center:M](M) {};
                \node [block, right = 1 of M, label=center:Y](Y) {};
                \draw[edge, dotted] (Z) to (W);
                \draw[edge, dotted] (Z) to (X);
                \draw[edge, dotted] (M) to (W);
                \draw[edge, dotted] (M) to (Y);
            \end{tikzpicture}
            \caption{Path $\PR$ when it has one collider, denoted by $W$.
            In this path, variable $M$ can coincide with $Y$.}
            \label{fig:collider_path}
        \end{subfigure}
        \caption{The two cases for $\PR$.}
        \label{fig: p-path}
    \end{figure}

    Figures \ref{fig:non-collider_path} and \ref{fig:collider_path} depict $\PR$ when it has zero or one collider, respectively.
    In these figures, directed edges indicate the presence of a directed path.
    We note that in Figure \ref{fig:collider_path}, $M$ can coincide with $Y$, but variable $Z$ is distinct from $X$ and $W$.
    This is because $W$ is a collider on this path, and $\PR$ exists in graph $\Grs_{\underline{X}}$ and should have an arrowhead towards $X$.
    Moreover, in Figure \ref{fig:non-collider_path}, $Z$ cannot coincide with neither $X$ (since $\PR$ exists in $\Grs_{\underline{X}}$) nor $Y$ (since $Y \notin \Anc{X}{\Gr}$).
    
    Let $\Gr'$ be a minimal (in terms of edges) subgraph of $\Grs$ such that
    \begin{enumerate*}[label=(\roman*)]
        \item $\Gr'$ contains $\PR$,
        \item $X \in \Anc{S}{\Gr'}$, and
        \item if $\mathcal{P}$ has exactly one collider $W$, then $W \in \Anc{S}{\Gr'}$.
    \end{enumerate*}
    Note that $W$ is an ancestor of $S$ in $\Grs$ since $S$ does not block $\mathcal{P}$, thus $\Gr'$ exists.
    Due to the properties of $\Gr'$, we have
    \begin{equation*}
         X \in \Anc{S}{\Gr'} \quad \text{ and } \quad (X \notindependent Y \vert S)_{\Gr'_{\underline{X}}}.
    \end{equation*}
    Denote by $\PR_{X\to S}$, the directed path from $X$ to $S$ in $\Gr'$.
    Note that this path is unique because of the minimality of $\Gr'$.
    Similarly, if $\PR$ has a collider (Figure \ref{fig:collider_path}), denote by $\PR_{W\to S}$, the directed path from $W$ to $S$ in $\Gr'$.

    Next, we show that in the case of Figure \ref{fig:collider_path}, $\PR_{W\to S}$ does not share any edges with $\PR$.
    We denote by $\PR_{X \gets Z}$, $\PR_{Z \to W}$, $\PR_{W \gets M}$, and $\PR_{M \to Y}$ the directed paths that create $\PR$ in Figure \ref{fig:collider_path}.
    Since $\Gr$ does not have any cycles, then $\PR_{W \to S}$ cannot share any edges with $\PR_{Z \to W}$ and $\PR_{W \gets M}$.
    If $\PR_{W \to S}$ has an intersection with $\PR_{M \to Y}$ in a node $T$, then consider the following path: $(\PR_{X \gets Z}, \PR_{Z \to W}, \PR_{W \to T}, \PR_{T \to Y})$.
    This path does not have any colliders, and $S$ does not block it, which contradicts the minimality (in terms of number of colliders) of $\PR$.
    Similarly, if $\PR_{W \to S}$ has an intersection with $\PR_{X \gets Z}$ in a node $T$, then $(\PR_{X \gets T}, \PR_{T \gets W}, \PR_{W \gets M}, \PR_{M \to Y})$ does not have any colliders and $S$ does not block it, which is again a contradiction.
    Therefore, $\PR_{W \to S}$ does not share any edges with $\PR$.

    In both cases in Figure \ref{fig: p-path}, $\PR_{X \to S}$ can overlap with $\PR$.
    Thus, $\PR_{X \to S}$ can go through $Y$.
    Moreover, $Y$ can coincide with $M$.
    Hence, $Y$ can have children in $\Gr'$.
    In the following, we use Lemma \ref{lemma:y_father} to show that it suffices to prove that $\prs_X(Y)$ is not s-ID only for those $\Gr'$ in which $Y$ has at least a child.
    Suppose $\Ch{Y}{\Gr'} = \varnothing$, i.e., $\PR_{X \to S}$ does not pass through $Y$, and in the case of Figure $\ref{fig:collider_path}$, $M$ is distinct from $Y$.
    In this case, $Y$ has a unique parent in $\Gr'$, denoted by $Y'$.
    Note that $Y'$ exists in $\PR$.
    Due to Lemma \ref{lemma:y_father}, it suffices to show that $\prs_X(Y')$ is not s-ID in the graph obtained by removing $Y$.
    Note that after removing $Y$ from $\PR$ and $\Gr'$, the new structure is again in the form of Figure $\ref{fig: p-path}$, where $Y$ is replaced by $Y'$.
    Now, if $Y'$ does not have any children in the remaining graph, we can repeat the same process until the graph simplifies to a structure in which the target variable has at least one child. 
    Nevertheless, it suffices to only consider the cases in which $Y$ has a child.
    Hence, for the rest of the proof, suppose $Y$ has at least one child in $\Gr'$.

    Next, we introduce a subgraph $\Gr''$ of $\Gr'$ that is of the form of a graph depicted in Figure \ref{fig: proof theorem 1}.
    This figure illustrates two types of DAGs, where the dotted edges indicate the presence of a directed path, and the directed paths do not share any edges.
    Variable $N$ can coincide with $Y$ in Figure \ref{fig: proof type1} and with $S$ in Figure \ref{fig: proof type2}.
    Furthermore, in Figure \ref{fig: proof type2}, the directed path in red is towards a variable highlighted in the figure, i.e., the variables in the directed paths from $Z$ to $N$ (except $Z$ itself), from $N$ to $S$, and from $Y$ to $N$.
    
    To introduce $\Gr''$, we consider the two cases for $\PR$ (as depicted in Figure $\ref{fig: p-path}$) separately.
    
    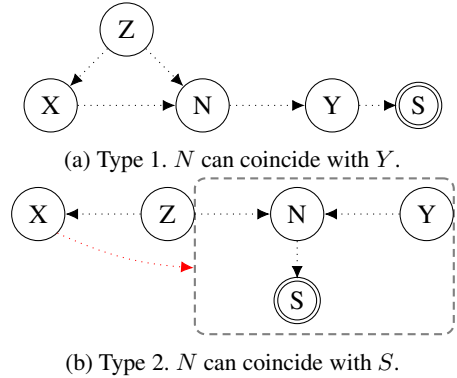
\begin{figure}[ht]
        \centering
        \begin{subfigure}[b]{1\textwidth}
            \centering
            \begin{tikzpicture}
                \tikzset{edge/.style = {->,> = latex',-{Latex[width=1.5mm]}}}
                \node [block, label=center:Z](Z) {};
                \node [block, below right = 0.5 and 0.5 of Z, label=center:N](N) {};
                \node [block, right = 1 of N, label=center:Y](Y) {};
                \node [block, below left = 0.5 and 0.5 of Z, label=center:X](X) {};
                \node [block-s, right = 0.5 of Y, label=center:S](S) {};
                \draw[edge, dotted] (Z) to (N);
                \draw[edge, dotted] (N) to (Y);
                \draw[edge, dotted] (X) to (N);
                \draw[edge, dotted] (Y) to (S);
                \draw[edge, dotted] (Z) to (X);
            \end{tikzpicture}
            \caption{Type 1. $N$ can coincide with $Y$.}
            \label{fig: proof type1}
        \end{subfigure}
        
        \begin{subfigure}[b]{1\textwidth}
            \centering
            \begin{tikzpicture}[node distance = 1ex and 0em, outer/.style={draw=gray, thick, rounded corners,
                          densely dashed, fill=white!5,
                          inner xsep=0ex, xshift=0ex, inner ysep=2ex, yshift=1ex,
                          fit=#1}]
                    \tikzset{edge/.style = {->,> = latex',-{Latex[width=1.5mm]}}}
                    \node [block, label=center:Z](Z) {};
                    \node [block, right = 1 of Z, label=center:N](N) {};
                    \node [block, right = 1 of N, label=center:Y](Y) {}; 
                    \node [block, left = 1 of Z, label=center:X](X) {};
                    \node [block-s, below = 0.5 of N, label=center:S](S) {};
                    \draw[edge, dotted] (Z) to (N);
                    \draw[edge, dotted] (Y) to (N);
                    \draw[edge, dotted] (N) to (S);
                    \draw[edge, dotted] (Z) to (X);
                    \scoped[on background layer]
                    \node [outer=(Z.east)(S) (Y.east),
                         label={[anchor=north]:}](M) {};
                    \draw[edge, dotted, red] (X.south east) to [bend right=10] (M);

            \end{tikzpicture}
            \caption{Type 2. $N$ can coincide with $S$.}
            \label{fig: proof type2}
            \end{subfigure}
        \caption{DAG $\Gr''$, a subgraph of $\Gr'$.
        The dotted edges indicate the presence of a directed path.} 
        \label{fig: proof theorem 1}
    \end{figure}

    \textit{Type 1: When $\PR$ has no collider.}
    In this case, $\PR$ is in the form of Figure \ref{fig:non-collider_path}.
    Since $Y$ has a child in $\Gr'$, $\PR_{X\to S}$ must pass through $Y$.
    Denote by $N$ the first intersection of $\PR_{X\to S}$ with $\PR$.
    Note that $N$ belongs to $\PR_{Z \to Y}$.
    This shows that $\Gr'$ must be in the form of Figure \ref{fig: proof type1}.
    Thus, we set $\Gr''$ to be $\Gr'$ itself in this case.

    \textit{Type 2: When $\PR$ has one collider.}
    In this case, $\PR$ is in the form of Figure \ref{fig:collider_path}.
    Since $Y$ has a child in $\Gr'$, either $Y$ coincide with $M$ in Figure \ref{fig:collider_path} or $\PR_{X\to S}$ passes through $Y$.
    In the first case, $\Gr'$ itself is in the form of Figure \ref{fig: proof type2}, and we set $\Gr''$ to be $\Grs$.
    In the second case, let $\Gr''$ be the subgraph of $\Gr'$ consisting of the edges of $\PR_{X \gets Z}$, $\PR_{Z \to W}$, $\PR_{W \to S}$, and $\PR_{X \to S}$.
    It is straightforward to see that $\Gr''$ is in the form of Figure \ref{fig: proof type2}, where the red edges are towards $Y$.
    
    So far, we have shown that to complete our proof, it suffices to show that $\interSp$ is not s-ID in $\Gr''$, where $\Gr''$ is in the form of a DAG depicted in Figure \ref{fig: proof theorem 1}.
    In the next part, we introduce two SEMs with causal graph $\Gr''$, denoted by $\M_1$ and $\M_2$.
    Then, we show that $\pro^{\M_1}(\Vb \vert S = 1) = \pro^{\M_2}(\Vb \vert S = 1)>0$.
    Finally, in the last part, we show that $\pro^{\M_1}_{X}(Y \vert S = 1) \neq \pro^{\M_2}_{X} (Y \vert S = 1)$.
    This shows that $\interSp$ is not s-ID in $\Gr''$ and completes the proof.

\subsubsection{Constructing SEMs $\M_1$ and $\M_2$ with Causal Graph $\Gr''$.}
    Suppose $\Gr''$ is in the form of a DAG depicted in either Figure \ref{fig: proof type1} or \ref{fig: proof type2}.
    For ease of notation, denote by $\Vb$, the set of variables of $\Gr''$  (excluding $S$) depicted in Figure \ref{fig: proof theorem 1}.
    In this part, we introduce two SEMs $\M_1$ and $\M_2$ with causal graph $\Gr''$.

    Consider the path from $Z$ to $X$ and denote the child of $Z$ in this path by $Z'$ ($Z'$ can be $X$).
    Note that $Z'$ does not have any other parent in $\Gr''$ due to the structure of $\Gr''$.
    For brevity of notations, we introduce the following definition.
    \begin{definition}
        For each variable $V$, we define the random variable $\sump{V}$ as the sum of the parents of $V$ in $\Gr''$.
        That is, 
        \begin{equation*}
            \sump{V} \coloneqq \sum_{W \in \Pa{V}{\Gr''}} W.
        \end{equation*}
    \end{definition}

    We now introduce SEM $\M_1$ in the following (we use $\N$ to denote a Gaussian distribution).

    \begin{equation} \label{eq: SEM1}
        \begin{split}
            &Z \sim \text{Bernoulli}(0.5) \\
            &Z' = -Z + \varepsilon_{Z}, \quad \varepsilon_{Z} \sim \N(1,1) \\
            &S \sim \text{Bernoulli} \left (\frac{1}{\sqrt{2\pi}}  \expo{-\frac{(\sump{S} + 1)^2}{2}} \right) \\
            &V = \sump{V} + \varepsilon_{V}, \quad \varepsilon_{V} \sim \N(1,1), \quad \forall V \in \Vb \setminus \{Z, Z'\}
        \end{split}
    \end{equation}

    Similarly, we define SEM $\M_2$ in the following.

    \begin{equation} \label{eq: SEM2}
        \begin{split}
            &Z \sim \text{Bernoulli}(0.5) \\
            &Z' = -Z + \varepsilon_{Z}, \quad \varepsilon_{Z} \sim \N(-1,1) \\
            &S \sim \text{Bernoulli} \left (\frac{1}{\sqrt{2\pi}}  \expo{-\frac{(\sump{S} - 1)^2}{2}} \right) \\
            &V = \sump{V} + \varepsilon_{V}, \quad \varepsilon_{V} \sim \N(-1,1), \quad \forall V \in \Vb \setminus \{Z, Z'\} 
        \end{split}
    \end{equation}

\subsubsection{Computing $\pro(\Vb \vert S = 1)$.}
    In this part, we show that 
    \begin{equation} \label{eq: proof 1, equal distributions}
        \pro^{\M_1}(\Vb \vert S = 1) = \pro^{\M_2} (\Vb \vert S = 1)>0,
    \end{equation}
    where SEMs $\M_1$ and $\M_2$ are defined in Equations \eqref{eq: SEM1} and \eqref{eq: SEM2}, respectively.
    It is straightforward to see that the corresponding distributions are positive.
    Thus, to show that Equation \eqref{eq: proof 1, equal distributions} holds, it suffices to show that 
    \begin{equation*}
        \pro^{\M_1} (\Vb, S = 1) = \pro^{\M_2} (\Vb, S = 1).
    \end{equation*}

    For both SEMs, using Markov factorization, we have
    \begin{align*}
        \pro(\Vb, S = 1) 
        &= \pro(S = 1 \vert \Pa{S}{\Gr''}) \prod_{V \in \Vb} \pro(V \vert \Pa{V}{\Gr''}) \\
        &= \pro(S = 1 \vert \Pa{S}{\Gr''}) \pro(Z) \pro(Z' \vert Z) \prod_{W  \in \Vb \setminus \{Z, Z'\}} \pro(W \vert \Pa{W}{\Gr''})\\
        &= 0.5 \pro(S = 1 \vert \Pa{S}{\Gr''})\pro(Z' \vert Z) \prod_{W  \in \Vb \setminus \{Z, Z'\}} \pro(W \vert \Pa{W}{\Gr''}).
    \end{align*}
    
    For each $W \in \Vb \setminus \{Z, Z'\}$, $W \vert \Pa{W}{\Gr''}$ is a Gaussian distribution. Therefore, for $i \in \{1,2\}$, we have 

    \begin{equation*}
        \pro^{\M_i}(W \vert \Pa{W}{}) = \frac{1}{\sqrt{2\pi}} \expo{-\frac{(W - \sump{W} - b_i)^2}{2}},
    \end{equation*}
    where $b_1 =1$ and $b_2=-1$.
    Similarly, for $Z'$ and $S$ we have
    \begin{equation*}
        \pro^{\M_i}(Z' \vert Z) = \frac{1}{\sqrt{2\pi}} \expo{-\frac{(Z'  + Z - b_i)^2}{2}},
    \end{equation*}
    \begin{equation*}
        \pro^{\M_i}(S = 1 \vert \Pa{S}{\Gr''}) 
        = \frac{1}{\sqrt{2\pi}}  \expo{-\frac{(\sump{S} + b_i)^2}{2}}.
    \end{equation*}
    
    Combing these equations, we have
    \begin{equation} \label{eq: proof, P(V, S)}
        \pro^{\M_i} (\Vb, S = 1)
        = \frac{1}{2 \times (2 \pi)^{\frac{|\Vb|}{2}}} \expo{\frac{-1}{2} \left [(\sump{S} + b_i)^2 + (Z' + Z - b_i)^2 + \mysum{W  \in \Vb \setminus \{Z,Z'\} }{} (W - \sump{W} - b_i)^2 \right ]},
    \end{equation}
    Note that $b_i^2= 1$.
    Thus, to show $\pro^{\M_1} (\Vb, S = 1) = \pro^{\M_2} (\Vb, S = 1)$, it suffices to show that the coefficient of $b_i$ in the exponent term of Equation \eqref{eq: proof, P(V, S)} is zero.
    That is, it suffices to show that 
    \begin{equation*}
        -\sump{S} + Z' + Z + \mysum{W  \in \Vb \setminus \{Z,Z'\}}{} (W - \sump{W})  = 0.
    \end{equation*}
    Since $\sump{Z'} = Z$ and $\sump{Z} = 0$, this is equivalent to show that 
    \begin{equation} \label{eq: proof 1, sum SPa}
        Z + \sum_{W \in \Vb} W = \sum_{W \in \Vb \cup \{S\}} \sump{W}.
    \end{equation}
    As depicted in Figure \ref{fig: proof theorem 1}, $Z$ has two children, $S$ has no child, and the rest of the variables has exactly one child in $\Gr''$.
    Hence, in the right-hand-side of Equation \eqref{eq: proof 1, sum SPa}, $Z$ appears twice, and each variable in $\Vb \setminus \{Z\}$ appears once.
    This proves that Equation \eqref{eq: proof 1, sum SPa} holds, and therefore, $\pro^{\M_1}(\Vb \vert S = 1) = \pro^{\M_2} (\Vb \vert S = 1)$.

\subsubsection{When $\Gr''$ is Type 1.}
    Suppose $\Gr''$ is in the form of Figure \ref{fig: proof type1}.
    In this part,  we show that there exists a real number $y$ such that 
    \begin{equation*}
        \pro_{X=0}^{\M_1}(y \vert S = 1) \neq \pro_{X=0}^{\M_2}(y \vert S = 1).
    \end{equation*}

    For both SEMs, we have 
    \begin{equation*}
        \pro_{X}(Y,S) 
        = \sum_{z\in \{0,1\}} \pro_{X}(Y, Z=z, S)
        = \sum_{z\in \{0,1\}} \pro_{X}(Y, S \vert Z=z) \pro_{X}(Z=z).
    \end{equation*}

    Applying Rule 2 of do-calculus on $\Gr''$, we have $\pro_{X}(Y, S \vert Z) = \pro(Y, S \vert X, Z)$.
    Also, since $X \notin \Anc{Z}{\Gr''}$, we have $\pro_{X}(Z) = \pro(Z)=0.5$, thus
    
    \begin{align*}
        \pro_{X}(Y,S)
        &= \sum_{z\in \{0,1\}} \pro(Y, S \vert X, Z = z)  \pro(Z=z)  \\
        &= 0.5 \sum_{z\in \{0,1\}} \pro(S \vert X, Y, Z=z) \pro(Y \vert X, Z=z).
    \end{align*}

    Since $(S \independent \{X, Z\} \vert Y)_{\Gr''}$, we have $\pro(S \vert X, Y, Z) =  \pro(S \vert Y)$.
    Hence, 
    \begin{equation} \label{eq: proof, P_X(Y,S)}
        \pro_{X}(Y,S) 
        = 0.5 \pro(S \vert Y) \sum_{z\in \{0,1\}} \pro(Y \vert X, Z=z).
    \end{equation}

    Let $\Wb$ be the set of variables in $\PR_{Z \to N}$ except for $Z$, $\PR_{X \to N}$ except for $X$, and $\PR_{N \to Y}$.
    Further, denote by $m$ the number of variables in $\Wb$.
    Since each variable in $\Wb$ is equal to the sum of its parents and its exogenous noise, we have
    \begin{equation*}
        Y =  X + Z +  \mysum{W \in \Wb}{} \varepsilon_{W}.
    \end{equation*}
    Since exogenous noises are independent of each other, for SEM $\M_i$, we have 
    \begin{equation*}
        Y \vert  X = x, Z = z \sim \N(x + z + m b_i, m).
    \end{equation*}
    Recall that $b_1 = 1$ and $b_2=-1$.
    Similarly, we have
    \begin{equation*}
        \sump{S} \vert Y = y  \sim \N(y + (k-1)b_i, k-1),
    \end{equation*}
    where $k \geq 2$ is the number of variables in path $\PR_{Y \to S}$ (including $Y$ and $S$).
    By applying Lemma \ref{lemma:normal-bernoulli} (with $R = \sump{S}$, $b = b_i$, $\mu_{y} = y + (k-1)b_i$, and $\sigma^2 = k-1$), we have 
    \begin{equation*}
        S \vert Y = y \sim Bernoulli \left ( \frac{1}{\sqrt{2\pi k}} \expo{-\frac{(y + k b_i)^2}{2k}} \right).
    \end{equation*}
    Applying this to Equation \eqref{eq: proof, P_X(Y,S)}, we have
    \begin{equation} \label{eq: P^Mi_X(Y,S)}
        \begin{split}
            \pro^{\M_i}_{X = 0}(Y=y,S = 1) 
            &= C \expo{-\frac{(y + k b_i)^2}{2k}} \sum_{z\in \{0,1\}} \expo{-\frac{(y - z - m b_i)^{2}}{2m}}\\
            &= C \expo{-\frac{(y + k b_i)^2}{2k}} \left(\expo{-\frac{(y- mb_i)^{2}}{2m}} + \expo{-\frac{(y - 1 - mb_i)^{2}}{2m}} \right),
        \end{split}
    \end{equation}
    where $C = 0.5 \times \frac{1}{\sqrt{2\pi k}} \times \frac{1}{\sqrt{2\pi m}} = \frac{1}{4\pi \sqrt{k m}}$ is a constant.

    Lemma \ref{lemma:lem_point} shows that for any real numbers $y_0$ and $y_1$, if $\prs_X(Y)$ is s-ID in $\Gr''$, then we have
    \begin{equation*}
        \frac{\pro^{\M_1}_{X}(Y = y_0, S = 1)}{\pro^{\M_2}_{X}(Y = y_0, S = 1)} 
        = \frac{\pro^{\M_1}_{X}(Y = y_1, S = 1)}{\pro^{\M_2}_{X}(Y = y_1 , S = 1)}.
    \end{equation*}
    Nevertheless, to conclude this part, we will demonstrate that 
    \begin{equation} \label{eq: ratio not equal}
        \frac{\pro^{\M_1}_{X=0}(Y = 0, S = 1)}{\pro^{\M_2}_{X=0}(Y = 0, S = 1)} 
        \neq \frac{\pro^{\M_1}_{X=0}(Y = 1, S = 1)}{\pro^{\M_2}_{X=0}(Y = 1, S = 1)}.
    \end{equation}
    Using Equation \eqref{eq: P^Mi_X(Y,S)}, we have
    \begin{align*}
        r \coloneqq
        \frac{\pro^{\M_1}_{X=0}(Y = 0, S = 1)}{\pro^{\M_2}_{X=0}(Y = 0, S = 1)} 
        &= \frac{\expo{-\frac{m}{2}} +  \expo{-\frac{(m + 1)^2}{2m}}}{ \expo{-\frac{m}{2}} +  \expo{-\frac{(m - 1)^2}{2m}}}\\
        &= \frac{1 + \expo{-1 - \frac{1}{2m}}}{1 + \expo{1 - \frac{1}{2m}}}.
    \end{align*}
    Similarly, we have
    \begin{align*}
        \frac{\pro^{\M_1}_{X=0}(Y = 1, S = 1)}{\pro^{\M_2}_{X=0}(Y = 1, S = 1)}
        &= \frac{\expo{-\frac{(k + 1)^2}{2k}} \left(\expo{-\frac{(m - 1)^2}{2m}} +  \expo{-\frac{m}{2}}\right)}{\expo{-\frac{(k - 1)^2}{2k}} \left(\expo{-\frac{(m + 1)^2}{2m}} +  \expo{-\frac{m}{2}}\right)}\\
        &= e^{-2} \times \frac{1 + \expo{1-\frac{1}{2m}}}{1 + \expo{-1-\frac{1}{2m}}}
        = e^{-2} r^{-1}.
    \end{align*}

    Lemma \ref{lem: inequality} implies that $r>e^{-1}$.
    However, $r = e^{-2} r^{-1}$ if and only if $r= \pm e^{-1}$.
    This shows that Equation \eqref{eq: ratio not equal} holds.
    Hence, as we mentioned before, Lemma \ref{lemma:lem_point} implies that $\prs_X(Y)$ is not s-ID in $\Gr''$.

\subsubsection{When $\Gr''$ is Type 2, and $X\in \Anc{Y}{\Gr''}$.}
    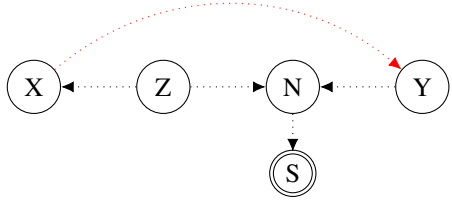
\begin{figure}[ht]
        \centering
        \begin{tikzpicture}[
            node distance = 1ex and 0em,
            outer/.style={draw=gray, thick, rounded corners,
                  densely dashed, fill=white!5,
                  inner xsep=1ex, xshift=1ex, inner ysep=2ex, yshift=1ex,
                  fit=#1}
           ]
            \tikzset{edge/.style = {->,> = latex',-{Latex[width=1.5mm]}}}
            \node [block, label=center:Z](Z) {};
            \node [block, right = 1 of Z, label=center:N](N) {};
            \node [block, right = 1 of N, label=center:Y](Y) {}; 
            \node [block, left = 1 of Z, label=center:X](X) {};
            \node [block-s, below = 0.5 of N, label=center:S](S) {};
            \draw[edge, dotted] (Z) to (N);
            \draw[edge, dotted] (Y) to (N);
            \draw[edge, dotted] (N) to (S);
            \draw[edge, dotted] (Z) to (X);
            \draw[edge, dotted, red] (X) to [bend left=40] (Y);
        \end{tikzpicture}
        \caption{When $\Gr''$ is Type 2, and $X\in \Anc{Y}{\Gr''}$. $N$ can coincide with $S$.}
        \label{fig: proof type2-anc}
    \end{figure}

    Suppose $\Gr''$ is in the form of Figure \ref{fig: proof type2} and $X\in \Anc{Y}{\Gr''}$.
    This happens when the red directed path in Figure \ref{fig: proof type2} is towards $Y$. Therefore $\Gr''$ has the form of Figure \ref{fig: proof type2-anc}. Similar to the previous part, we show that there exists a real number $y$ such that 
    \begin{equation*}
        \pro_{X=0}^{\M_1}(y \vert S = 1) \neq \pro_{X=0}^{\M_2}(y \vert S = 1).
    \end{equation*}
    We have 
    \begin{align*}
        \pro_{X}(Y,S) 
        &= \sum_{z\in\{0,1\}} \pro_{X}(Y, Z=z S)\\
        &= \sum_{z\in\{0,1\}} \pro_{X}(Y, S \vert Z=z) \pro_{X}(Z=z).
    \end{align*}
    Since $X \notin \Anc{Z}{\Gr''}$, we have $\pro_{X}(Z=z) = \pro(Z=z)=0.5$.
    Further, since $(X \independent Y,S \vert Z)_{\Gr_{\underline{X}}''}$, using Rule 2 of do-calculus, we have $\pro_{X}(Y, S \vert Z) = \pro(Y, S \vert X, Z)$.
    Hence, 
    \begin{align*}
        \pro_{X}(Y,S)
        &= 0.5 \sum_{z\in\{0,1\}} \pro(Y, S \vert X, Z=z) \\
        &= 0.5 \sum_{z\in\{0,1\}} \pro(S \vert X, Y, Z=z) \pro(Y \vert X, Z=z).
    \end{align*}
    Furthermore, since $(Y \independent Z \vert X)_{\Gr''}$ and $(X \independent S \vert Y, Z)_{\Gr''}$, we have
    \begin{equation*}
        \pro_{X}(Y,S) = 0.5 \pro(Y \vert X) \sum_{z\in\{0,1\}} \pro(S \vert Y, Z=z).
    \end{equation*}
    
    Similar to the previous part, for SEM $\M_i$, we have
    \begin{equation*}
        Y \vert X = x \sim \N(x + k b_i, k),
    \end{equation*}
    \begin{equation*}
        S \vert Y = y, Z = z \sim \text{Bernoulli} \left ( \frac{1}{\sqrt{2\pi m}} \expo{-\frac{(y + z + mb_i)^2}{2m}} \right ),
    \end{equation*}
    where $b_1=1$, $b_2=-1$, $k+1$ is the number of variables in the directed path $\PR_{X \to Y}$, and $m+2$ is the number of variables of the union of paths $\PR_{Z\to N}, \PR_{N\gets Y}, \PR_{N\to S}$.
    Therefore, we have
    \begin{equation*}
        \pro_{X = 0}(Y=y, S = 1) 
        = C \expo{-\frac{(y - kb_i)^2}{2k}} \left( \expo{-\frac{(y + mb_i)^2}{2m}} + \expo{-\frac{(y + mb_i + 1)^2}{2m}} \right),
    \end{equation*}
    where $C = 0.5 \times \frac{1}{\sqrt{2\pi k}} \times \frac{1}{\sqrt{2\pi m}} = \frac{1}{4\pi \sqrt{k m}}$ is a constant.
    This is similar to Equation \eqref{eq: P^Mi_X(Y,S)}, when $y$ is replaced by $-y$.
    Therefore, with similar calculations, we have
    \begin{equation*}
        \frac{\pro^{\M_1}_{X=0}(Y = 0, S = 1)}{\pro^{\M_2}_{X=0}(Y = 0, S = 1)} 
        \neq \frac{\pro^{\M_1}_{X=0}(Y = -1, S = 1)}{\pro^{\M_2}_{X=0}(Y = -1, S = 1)}.
    \end{equation*}
    Hence, Lemma \ref{lemma:lem_point} implies that $\prs_X(Y)$ is not s-ID in $\Gr''$.

\subsubsection{When $\Gr''$ is Type 2, and $X\notin \Anc{Y}{\Gr''}$.}

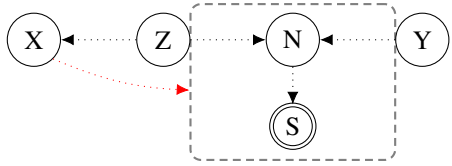
\begin{figure}[ht]
    \centering
    \begin{tikzpicture}[
        node distance = 1ex and 0em,
        outer/.style={draw=gray, thick, rounded corners,
              densely dashed, fill=white!5,
              inner xsep=0ex, xshift=0ex, inner ysep=2ex, yshift=1ex,
              fit=#1}
       ]
        \tikzset{edge/.style = {->,> = latex',-{Latex[width=1.5mm]}}}
        \node [block, label=center:Z](Z) {};
        \node [block, right = 1 of Z, label=center:N](N) {};
        \node [block, right = 1 of N, label=center:Y](Y) {}; 
        \node [block, left = 1 of Z, label=center:X](X) {};
        \node [block-s, below = 0.5 of N, label=center:S](S) {};
        \draw[edge, dotted] (Z) to (N) ;
        \draw[edge, dotted] (Y) to (N);
        \draw[edge, dotted] (N) to (S);
        \draw[edge, dotted] (Z) to (X);
        \scoped[on background layer]
        \node [outer=(Z.east)(S) (Y.west),
             label={[anchor=north]:}](M) {};
        \draw[edge, dotted, red] (X.south east) to [bend right=10] (M);
\end{tikzpicture}

   \caption{When $\Gr''$ is Type 2, and $X\notin \Anc{Y}{\Gr''}$. $N$ can coincide with $S$.}
\label{fig: proof type2-not-anc}
\end{figure}

    In this case, we have Figure \ref{fig: proof type2-not-anc}, in which the red directed path is not towards $Y$, and $Y$ does not have any parents in $\Gr''$.
    Thus, $(Y \independent Z, X)_{\Gr''}$.
    With similar calculations as the previous parts, we have
    \begin{align*}
        \pro_{X}(Y,S) 
        &= \sum_{z\in\{0,1\}} \pro_{X}(Y, Z=z, S) \\
        &= \sum_{z\in\{0,1\}} \pro_{X}(Y, S \vert Z=z) \pro_{X}(Z=z) \\
        &= \sum_{z\in\{0,1\}} \pro(Y, S \vert Z=z, X) \pro(Z)  \\
        &= 0.5 \sum_{z\in\{0,1\}} \pro(S \vert X, Y, Z=z) \pro(Y \vert Z=z, X)\\
        &= 0.5\pro(Y) \sum_{z\in\{0,1\}} \pro(S \vert X, Y, Z=z).
    \end{align*} 
    Since $Y$ does not have any parents, in $\M_i$ we have
    \begin{equation*}
        Y \sim \N(b_i, 1),
    \end{equation*}
    where $b_1=1$ and $b_2=-1$.
    Analogous to the previous parts, we have
    \begin{equation*}
        S \vert X = x, Y = y, Z = z \sim \text{Bernoulli} \left ( \frac{1}{\sqrt{2\pi m}} \expo{-\frac{(x + y + z + mb_i)^2}{2m}} \right ),
    \end{equation*}
    where $m$ is a constant (indeed, $m$ is equal to $\vert \Vb \vert$ minus the number of nodes in path $\PR_{X \gets Z}$).
    Therefore, we have
    \begin{equation*}
        \pro_{X = 0}(Y=y, S = 1) 
        = C \expo{-\frac{(y - b_i)^2}{2}} \left( \expo{-\frac{(y + m b_i)^2}{2m}} + \expo{-\frac{(y + m b_i + 1)^2}{2m}} \right),
    \end{equation*}
     where $C = \frac{1}{4\pi\sqrt{m}}$.
    This is again in the form of Equation \eqref{eq: P^Mi_X(Y,S)} when $k=1$ and $y$ is replaced by $-y$.
    Therefore, with similar calculations, we have
    \begin{equation*}
        \frac{\pro^{\M_1}_{X=0}(Y = 0, S = 1)}{\pro^{\M_2}_{X=0}(Y = 0, S = 1)} 
        \neq \frac{\pro^{\M_1}_{X=0}(Y = -1, S = 1)}{\pro^{\M_2}_{X=0}(Y = -1, S = 1)}.
    \end{equation*}
    Hence, Lemma \ref{lemma:lem_point} implies that $\prs_X(Y)$ is not s-ID in $\Gr''$.

\subsection{Proof of Proposition \ref{prp: X2}}
    \noindent{\bf Proposition  \ref{prp: X2}.} \textit{
        Suppose $\Grs$ is an augmented DAG over $\Vb \cup \{S\}$, and let $\Xb \subsetneq \Vb$.
        For $\Xb_{2} \coloneqq \Xb \setminus \Anc{S}{}$, conditional causal effect $\prs_{\Xb_2} (\Vb \setminus \Xb_2)$ is s-ID in $\Grs$ and can be computed from $\prs (\Vb)$ by
        \begin{equation} \label{eq: proof X2}
            \prs (\Anc{S}{} \setminus S) \prod_{W \in \Wb} \prs (W \vert \Pa{W}{}),
        \end{equation}
        where $\Wb =  \Vb \setminus (\Xb_2 \cup \Anc{S}{})$.
    }

    \begin{proof}
        Since the variables in $\Xb_2$ are not ancestors of $S$, Rule 3 of do-calculus implies that $\pro_{\Xb_2}(S) = \pro(S)$.
        Thus, using truncated Markov factorization, we have
        \begin{equation} \label{eq: proof prp}
            \begin{split}
                \pro_{\Xb_2}(\Vb \setminus \Xb_2 \vert S) 
                &= \frac{\pro_{\Xb_2}(\Vb \setminus \Xb_2, S)}{\pro_{\Xb_2}(S)} 
                = \frac{\pro_{\Xb_2} (\Vb \setminus \Xb_2, S)}{\pro(S)} \\
                &= \frac{1}{\pro(S)}  \prod_{V \in  (\Vb \cup \{S\}) \setminus \Xb_2 } \pro (V \vert \Pa{V}{}) \\
                &= \frac{1}{\pro(S)}  \prod_{V \in \Anc{S}{}} \pro (V \vert \Pa{V}{}) \prod_{W \in \Wb} \pro (W \vert \Pa{W}{}).
            \end{split}
        \end{equation}
        Applying Lemma \ref{lem: ancestral} for $\Vb' = \Anc{S}{}$ in $\Grs$, we have 
        \begin{equation} \label{eq:ancestor_factorization}
            \pro(\Anc{S}{}) =  \prod_{V \in \Anc{S}{}} \pro (V \vert \Pa{V}{}).
        \end{equation}
        Since $\Wb \cap \Anc{S}{} = \varnothing$, for any $W \in \Wb$, we have $(W \independent S \vert \Pa{W}{})_{\Gr}$, thus
        \begin{equation} \label{eq: proof W factorization}
            \pro (W \vert \Pa{W}{})=  \pro (W \vert \Pa{W}{}, S). 
        \end{equation}
        Finally, by substituting Equations \eqref{eq:ancestor_factorization} and \eqref{eq: proof W factorization} in Equation \eqref{eq: proof prp}, we have
        \begin{align*}
            \prs_{\Xb_2}(\Vb \setminus \Xb_2) 
            &= \pro_{\Xb_2}(\Vb \setminus \Xb_2 \vert S=1) \\
            &=\frac{\pro(\Anc{S}{}\setminus \{S\}, S=1 )}{\pro(S=1)} \prod_{W \in \Wb} \pro (W \vert \Pa{W}{}, S=1)\\
            &= \prs (\Anc{S}{} \setminus S) \prod_{W \in \Wb} \prs (W \vert \Pa{W}{}).
        \end{align*}
        This shows that $\prs_{\Xb_2}(\Vb \setminus \Xb_2)$ is s-ID in $\Grs$ and concludes the proof.  
    \end{proof}
    
\subsection{Proof of Theorem \ref{th:markov-general}}
    \textbf{Theorem 2.}\textit{
        For disjoint subsets $\Xb$ and $\Yb$ of $\Vb$, let $\Xb_{1} \coloneqq \Xb \cap \Anc{S}{}$ and $\Xb_{2} \coloneqq \Xb \setminus \Anc{S}{}$.
        \begin{itemize}
            \item 
                If $\Xb_1 = \varnothing$:  Conditional causal effect $\prs_{\Xb} (\Yb)$ is s-ID and can be computed from Equation \eqref{eq: X2 to Y}.
            \item 
                If $\Xb_1 \neq \varnothing$: Conditional causal effect $\prs_{\Xb} (\Yb)$ is s-ID if and only if
                \begin{equation} \label{eq: proof general case cond}
                    (\Xb_{1} \independent \Yb \vert \Xb_{2}, S)_{\Grs_{\underline{\Xb_{1}}\overline{\Xb_{2}}}}.
                \end{equation}
                Moreover, when \eqref{eq: proof general case cond} holds, we have 
                \begin{equation} \label{eq: proof general_formula}
                    \prs_{\Xb} (\Yb) = \prs (\Xb_1)^{-1} \sum_{\Vb \setminus (\Xb \cup \Yb)} \prs_{\Xb_2} (\Vb \setminus \Xb_2),
                \end{equation}
                where $\prs_{\Xb_2} (\Vb \setminus \Xb_2)$ can be computed from $\prs(\Vb)$ using Equation \eqref{eq: proof X2}.
        \end{itemize}
    }

\begin{proof}
    
    The first part of the theorem (if $\Xb_1 = \varnothing$) is a direct consequence of Proposition 1.
    Suppose $\Xb_1 \neq \varnothing$.
        
    \subsubsection{Sufficiency.}
        Suppose Equation \eqref{eq: proof general case cond} holds.
        Hence, by applying Rule 2 of do-calculus, we have
        \begin{equation} \label{eq: proof sufficiency 1}
            \pro_{\Xb}(\Yb \vert S) = \pro_{\Xb_{2}}(\Yb \vert \Xb_{1}, S).
        \end{equation}
        The variables in $\Xb_2$ are not ancestors of $\Xb_1 \cup \{S\}$, thus $(\Xb_2 \independent \Xb_1 \vert S)_{\Gr_{\overline{\Xb_2}}}$.
        Accordingly, by applying Rule 3 of do-calculus, we have
        \begin{equation} \label{eq: proof sufficiency 2}
            \pro_{\Xb_{2}}(\Xb_{1} \vert S) = \pro(\Xb_{1} \vert S).
        \end{equation}
        Using Equations \eqref{eq: proof sufficiency 1} and \eqref{eq: proof sufficiency 2}, we have
        \begin{equation} \label{eq: proof sufficiency 3}
             \prs_{\Xb}(\Yb)
             = \prs_{\Xb_{2}}(\Yb \vert \Xb_{1}) 
             = \frac{\prs_{\Xb_{2}}(\Xb_{1}, \Yb)}{\prs_{\Xb_{2}}(\Xb_{1})} 
             = \frac{\prs_{\Xb_{2}}(\Xb_{1}, \Yb)}{\prs(\Xb_{1})}.
        \end{equation}
        Moreover, Corollary \ref{coro} for $\Xb_{1} \cup \Yb$ implies
        \begin{equation} \label{eq: proof sufficiency 4}
            \prs_{\Xb_{2}}(\Xb_{1}, \Yb) = \sum_{\Vb \setminus (\Xb \cup \Yb)} \prs_{\Xb_2} (\Vb \setminus \Xb_2).
        \end{equation}
        Equation \eqref{eq: proof general_formula} can be obtained by merging Equations \eqref{eq: proof sufficiency 3} and \eqref{eq: proof sufficiency 4}.

    \subsubsection{Necessity.}
        Suppose $(\Xb_{1} \notindependent \Yb \vert \Xb_{2}, S)_{\Grs_{\underline{\Xb_{1}}\overline{\Xb_{2}}}}$.
        We need to show that $\prs_{\Xb}(\Yb)$ is not s-ID in $\Grs$.

        Using Lemma \ref{lemma: G^*}, which we provide at the end of this proof, there exists $X^* \in \Xb_1$, $Y^* \in \Yb$, and a subgraph $\Gr^*$ of $\Grs$ such that
        \begin{itemize}
            \item $\Xb \cap \Anc{S}{\Gr^*_{\overline{X^*}}} = \{X^*\}$,
            \item $(X^* \notindependent Y^* \vert S)_{\Gr^*_{\underline{X^*}}}$, and
            \item $(\Xb\setminus \{X^*\} \independent Y^* \vert X^*, S)_{\Gr^*_{\overline{\Xb}}}$.
        \end{itemize}      

        The first property implies that $X^* \in \Anc{S}{\Gr^*}$.
        Hence, Equation \eqref{eq: singleton} holds for $X^*$ and $Y^*$ in $\Gr^*$ and Theorem \ref{th:markov-single} implies that $\prs_{X^*}(Y^*)$ is not s-ID in $\Gr^*$.
        To conclude the proof using Lemma \ref{lemma:remove}, it suffices to show that $\interSbp$ is not s-ID in $\Gr^*$.
        
        For any SEM with causal graph $\Gr^*$, due to the first and third properties in Claim \ref{lemma:subgraph-non-id}, Rule 3 of do-calculus implies that
        \begin{equation*}
            \pro_{X^*}(Y^* \vert S) = \pro_{\Xb} (Y^* \vert S).
        \end{equation*}

        Therefore, in $\Gr^*$, the s-identifiability of $\pro^{s}_{X^*}(Y^*)$ is equivalent to s-identifiability of $\pro^{s}_{\Xb}(Y^*)$, thus $\pro^{s}_{\Xb}(Y^*)$ is not s-ID in $\Gr^*$.
        Therefore, Lemma \ref{lemma: Yb to Y} implies that $\pro^{s}_{\Xb}(\Yb)$ is not s-ID in $\Gr^*$ since $Y^* \in \Yb$, which concludes our proof.

\end{proof}

    \begin{lemma} \label{lemma: G^*}
        Let $\Grs$ be an augmented DAG over $\Vb \cup \{S\}$.
        For disjoint subsets $\Xb$ and $\Yb$ of $\Vb$, let $\Xb_{1} \coloneqq \Xb \cap \Anc{S}{}$ and $\Xb_{2} \coloneqq \Xb \setminus \Anc{S}{}$.
        If $\Xb_1 \neq \varnothing$ and $(\Xb_{1} \notindependent \Yb \vert \Xb_{2}, S)_{\Grs_{\underline{\Xb_{1}}\overline{\Xb_{2}}}}$, then there exists $X^* \in \Xb_1$, $Y^* \in \Yb$, and a subgraph $\Gr^*$ of $\Grs$ such that
        \begin{itemize}
            \item \textbf{Condition 1.}
                $\Xb \cap \Anc{S}{\Gr^*_{\overline{X^*}}} = \{X^*\}$,
            \item \textbf{Condition 2.}
                $(X^* \notindependent Y^* \vert S)_{\Gr^*_{\underline{X^*}}}$, and
            \item \textbf{Condition 3.}
                $(\Xb\setminus \{X^*\} \independent Y^* \vert X^*, S)_{\Gr^*_{\overline{\Xb}}}$.
        \end{itemize}
    \end{lemma}
    \begin{proof}
        For brevity of notations, let $\Gr_1$ denote $\Grs_{\underline{\Xb_{1}} \overline{\Xb_{2}}}$.
        Let $\PR$ be an active path with the minimum number of colliders between $\Xb_1$ and $\Yb$ in $\Gr_1$ that $\Xb_{2} \cup \{S\}$ does not block.
        Note that such path exists since $(\Xb_{1} \notindependent \Yb \vert \Xb_{2}, S)_{\Gr_1}$.
        Further, denote by $X \in \Xb_1$ and $Y^* \in \Yb$ the endpoints of $\PR$.

        We now show that $\PR$ contains no variables in $\Xb \setminus \{X\}$.
        The variables in $\Xb_2$ cannot be a collider in $\PR$ since $\PR$ exists in $\Gr_1$, where the incoming edges to $\Xb_2$ are removed.
        Further, since $\Xb_2 \cup \{S\}$ does not block $\PR$, then $\Xb_2$ cannot have non-colliders in $\PR$.
        Therefore, $\Xb_2$ has no variable in $\PR$.
        The variables in $\Xb_1 \setminus \{X\}$ cannot be a non-collider in $\PR$ since $\PR$ exists in $\Gr_1$, where the outgoing edges from $\Xb_1$ are removed.
        Since $\PR$ is an active path, each of its colliders must be in $\Anc{\Xb_2 \cup \{S\}}{\Gr_1}$.
        However, the set of descendants of $\Xb_1$ in $\Gr_1$ is $\Xb_1$ itself, thus, cannot have any descendants in $\Anc{\Xb_2 \cup \{S\}}{\Gr_1}$.
        Therefore, $\Xb_1 \setminus \{X\}$ cannot be a collider in $\PR$.
        This shows that $\PR$ has no variable from $\Xb \setminus \{X\}$.
        
        Next, we show that all the colliders of $\PR$ are ancestors of $S$ in $\Gr_1$.
        According to the definition of $d$-separation, the colliders of $\PR$ must be an ancestor of either $S$ or $\Xb_{2}$ in $\Gr_1$.
        Note that the variables in $\Xb_{2}$ do not have any incoming edges in $\Gr_1$.
        Therefore, the colliders of $\PR$ cannot be an ancestor of $\Xb_{2}$, thus, are ancestors of $S$ in $\Gr_1$.

        Next, similar to the proof of Lemma \ref{lemma:collider}, we show that $\PR$ has at most one collider.
        Assume, for the sake of contradiction, that $\PR$ has at least two colliders.
        Denote by $V_l$ and $V_r$ the closest colliders of $\PR$ to $X$ and $Y^*$, respectively.
        As we showed in the previous paragraph, both $V_l$ and $V_r$ must be ancestors of $S$ in $\Gr_1$.
        Accordingly, denote by $\PR_l$ and $\PR_r$, directed paths from $V_l$ to $S$ and from $V_r$ to $S$ in $\Gr_1$, respectively.
        Further, denote by $N$ the first intersection of $\PR_l$ and $\PR_r$ ($N$ can be $S$).
        Now, consider the path starting from $X$ to $V_l$ by the edges in $\PR$, then from $V_l$ to $N$ by the edges in $\PR_l$, then from $N$ to $V_r$ by the edges in $\PR_r$, and then from $V_r$ to $Y$ by the edges in $\PR$.
        In this path, $N$ is the only collider, and the non-colliders do not belong to $\Xb_2$.
        Note that $N$ is an ancestor of $S$ in $\Gr_1$.
        Therefore, $\Xb_2 \cup \{S\}$ does not block this path.
        This is a contradiction since $\PR$ is a path with the minimum number of colliders that satisfies this property.
        This shows that $\PR$ has at most one collider.

        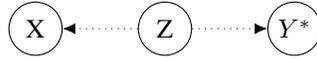
\begin{figure}[ht]
            \centering
            \begin{tikzpicture}
                \tikzset{edge/.style = {->,> = latex',-{Latex[width=1.5mm]}}}
                \node [block, label=center:Z](Z) {};
                \node [block, left = 1 of Z, label=center:X](X) {};
                \node [block, right = 1 of Z, label=center:$Y^*$](Y) {};
                \draw[edge, dotted] (Z) to (Y);
                \draw[edge, dotted] (Z) to (X);   
            \end{tikzpicture}
            \caption{Path $\PR$ when it has no colliders.}
            \label{fig: general1}
        \end{figure}    
    
        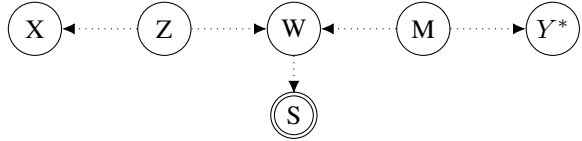
\begin{figure}[ht]
            \centering
            \begin{tikzpicture}
                \tikzset{edge/.style = {->,> = latex',-{Latex[width=1.5mm]}}}
                \node [block, label=center:Z](Z) {};
                \node [block, left = 1 of Z, label=center:X](X) {};
                \node [block, right = 1 of Z, label=center:W](W) {};
                \node [block, right = 1 of W, label=center:M](M) {};
                \node [block, right = 1 of M, label=center:$Y^*$](Y) {};
                \node [block-s, below = 0.5 of W, label=center:S](S) {};
                \draw[edge, dotted] (Z) to (W);
                \draw[edge, dotted] (Z) to (X);
                \draw[edge, dotted] (M) to (W);
                \draw[edge, dotted] (M) to (Y);
                \draw[edge, dotted] (W) to (S);          
            \end{tikzpicture}
            \caption{Paths $\PR$ and $\PR_{W \to S}$, when $\PR$ has one collider. In this graph, variable $M$ can coincide with $Y^*$.}
            \label{fig: general2}
        \end{figure}

        Figure \ref{fig: general1} depicts $\PR$ when it has no collider.
        Further, Figures \ref{fig: general2} depicts $\PR$ and $\PR_{W \to S}$, where $\PR$ has one collider and $\PR_{W \to S}$ denotes a directed path from $W$ to $S$ in $\Gr_1$.
        In the case of Figure \ref{fig: general2}, $M$ can coincide with $Y^*$.
        If $\PR$ has no collider, then we denote by $\Vb'$ the set of variables in $\PR$ excluding $X$.
        Further, if $\PR$ has one collider, then we denote by $\Vb'$ the set of variables in $\PR$ and $\PR_{W \to S}$, excluding $X$.
        As we showed before, recall that $\Vb' \cap \Xb = \varnothing$.

        Next, we show that in the case of Figure \ref{fig: general2}, i.e., when $\PR$ has one collider, $\PR_{W\to S}$ does not share any variable with $\PR$ other than $W$.
        We denote by $\PR_{X \gets Z}$, $\PR_{Z \to W}$, $\PR_{W \gets M}$, and $\PR_{M \to Y^*}$ the directed paths that create $\PR$ in Figure \ref{fig: general2}.
        Since $\Gr$ does not have any cycles, then $\PR_{W \to S}$ cannot share any edges with $\PR_{Z \to W}$ and $\PR_{W \gets M}$.
        If $\PR_{W \to S}$ has an intersection with $\PR_{M \to Y^*}$ in a node $T$, then consider the following path: $(\PR_{X \gets Z}, \PR_{Z \to W}, \PR_{W \to T}, \PR_{T \to Y^*})$.
        This path does not have any colliders, and $\Xb_2 \cup \{S\}$ does not block it in $\Gr_1$, which contradicts the minimality (in terms of number of colliders) of $\PR$.
        Similarly, if $\PR_{W \to S}$ has an intersection with $\PR_{X \gets Z}$ in a node $T$, then $(\PR_{X \gets T}, \PR_{T \gets W}, \PR_{W \gets M}, \PR_{M \to Y^*})$ does not have any colliders and $\Xb_2 \cup \{S\}$ does not block it in $\Gr_1$, which is again a contradiction.
        Therefore, $\PR_{W \to S}$ does not share any edges with $\PR$.
        
        We define $\Gr_2$ to be a minimal (in terms of edges) subgraph of $\Grs$ such that
        \begin{enumerate*}[label=(\roman*)]
            \item $\Gr_2$ contains $\PR$,
            \item $\Gr_2$ contains $\PR_{W \to S}$ if $\PR$ has a collider, and
            \item $X \in \Anc{S}{\Gr_2}$.
        \end{enumerate*}
        Note that such a graph exists since $X \in \Xb_1 \subseteq \Anc{S}{\Grs}$.
        We denote by $\PR_{X \to S}$, the directed path from $X$ to $S$ in $\Gr_2$.
        Furthermore, if $\PR_{X \to S}$ intersects with $\Vb'$, then we denote by $I$, the first intersection of $\PR_{X \to S}$ with $\Vb'$.
        Note that $I$ can be $S$.

        Note that in $\Gr_2$, the variables in $\Xb_2$ have no edges.
        This is because $\PR$ has no variables from $\Xb_2$, and the variables in $\Xb_2$ are not ancestors of $S$, thus, do not exist in the directed paths towards $S$.
        Therefore, there is no path between $\Xb_2$ and $Y^*$ in $\Gr_2$.

        To conclude the proof, we introduce $X^*, Y^*, \Gr^*$ and show that the three conditions of the lemma hold for them.

        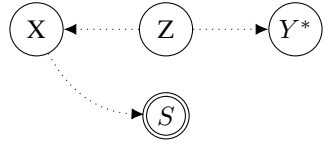
\begin{figure}[ht]
            \centering
            \begin{tikzpicture}
                \tikzset{edge/.style = {->,> = latex',-{Latex[width=1.5mm]}}}
                \node [block, label=center:Z](Z) {};
                \node [block, left = 1 of Z, label=center:X](X) {};
                \node [block, right = 1 of Z, label=center:$Y^*$](Y) {};
                \node [block-s, below = 0.5 of Z, label=center:$S$](S) {};
                \draw[edge, dotted] (Z) to (Y);
                \draw[edge, dotted] (Z) to (X);  
                \draw[edge, dotted] (X) to [bend right = 30] (S);   
            \end{tikzpicture}
            \caption{DAG $\Gr_2$ when $I$ does not exist.}
            \label{fig: no I}
        \end{figure}

        First, consider the case where $I$ does not exist, i.e., $\PR_{X \to S}$ does not intersect with $\Vb'$.
        In this case, $\Vb'$ does not include $S$, $\PR$ has no colliders, and $\PR_{X \to S}$ does not intersect with $\PR$ other than $X$.
        Accordingly, DAG $\Gr_2$ is in the form Figure \ref{fig: no I}.
        Let $\Gr^*$ be $\Gr_2$ and $X^*$ be the last variable in path $\PR_{X\to S}$ that is a member of $\Xb_1$.
        In $\Gr^*$, the directed paths from $\Xb \setminus \{X^{*}\}$ to $S$ pass through $X^*$, thus, Condition 1 holds.
        Moreover, the second condition holds since $S$ does not block path $(\PR_{X^* \gets X}, \PR)$ in $\Gr^*_{\underline{X^*}}$.
        Finally, as $\PR_{X \to S}$ does not intersect with $\Vb'$, there exists no path between $\Xb \setminus \{X^*\}$ and $Y^*$ in $\Gr^*_{\overline{\Xb}}$, thus, Condition 3 holds.

        Now, suppose $I$ exists.
        If path $\PR$ has no colliders, then $\Gr_2$ must be in the form of Figure \ref{fig: I exists, no collider}.
        \begin{figure}[ht]
            \centering
            \begin{tikzpicture}
                \tikzset{edge/.style = {->,> = latex',-{Latex[width=1.5mm]}}}
                \node [block, label=center:Z](Z) {};
                \node [block, left = 1 of Z, label=center:X](X) {};
                \node [block, right = 1 of Z, label=center:$I$](I) {};
                \node [block, right = 1 of I, label=center:$I'$](Ip) {};
                \node [block, right = 1 of Ip, label=center:$Y^*$](Y) {};
                \node [block-s, below = 0.5 of I, label=center:$S$](S) {};
                \draw[edge, dotted] (Z) to (X);
                \draw[edge, dotted] (Z) to (I); 
                \draw[edge, dotted] (I) to (Ip); 
                \draw[edge, dotted] (Ip) to (Y); 
                \draw[edge, dotted] (X) to [bend left = 40] (I);  
                \draw[edge, dotted] (Ip) to [bend left = 30] (S);  
            \end{tikzpicture}
            \caption{DAG $\Gr_2$ when $I$ exists and $\PR$ has no collider.}
            \label{fig: I exists, no collider}
        \end{figure}

        In the case of Figure \ref{fig: I exists, no collider}, path $\PR_{X\to S}$ is in the form $(\PR_{X \to I}, \PR_{I \to I'}, \PR_{I' \to S})$, where $I'$ can coincide with $I$ or $Y^*$.
        Note that directed paths do not overlap due to the minimality of $\Gr_2$.

        If path $\PR$ has one collider, then $\Gr_2$ must be in the form of Figure \ref{fig: I exists, one collider}.

        \begin{figure}[ht]
            \centering
            \begin{tikzpicture}[
                node distance = 1ex and 0em,
                outer/.style={draw=gray, thick, rounded corners,
                              densely dashed, fill=white!5,
                              inner xsep=0ex, xshift=0ex, inner ysep=2ex, yshift=1ex,
                              fit=#1}
                      ]
                \tikzset{edge/.style = {->,> = latex',-{Latex[width=1.5mm]}}}
                \node [block, label=center:Z](Z) {};
                \node [block, left = 1 of Z, label=center:X](X) {};
                \node [block, right = 1 of Z, label=center:W](W) {};
                \node [block, right = 1 of W, label=center:M](M) {};
                 
                \node [block, right = 1 of M, label=center:$I$](I) {};
    
                \node [block, right = 1 of I, label=center:$I'$](Ip) {};
                \node [block, right = 1 of Ip, label=center:$Y^*$](Y) {};
                \node [block-s, below = 0.5 of W, label=center:S](S) {};
                \draw[edge, dotted] (Z) to (W);
                \draw[edge, dotted] (Z) to (X);
                \draw[edge, dotted] (M) to (W);
                \draw[edge, dotted] (M) to (I);
                \draw[edge, dotted] (I) to (Ip);
                \draw[edge, dotted] (Ip) to (Y);
                \draw[edge, dotted] (W) to (S);
                
                \draw[edge, dotted] (X) to [bend left = 40]  (I);   
                \scoped[on background layer]
                \node [outer=(Z.east)(S) (M.west),
                     label={[anchor=north]:}](box) {};
                \draw[edge, dotted, red] (Ip) to [bend left=20] (box);
            \end{tikzpicture}
            \caption{DAG $\Gr_2$ when $I$ exists and $\PR$ has one collider.}
            \label{fig: I exists, one collider}
        \end{figure}

        Similarly, in this figure, $I'$ can coincide with $I$ or $Y^*$.
        The red directed path is from $I'$ to a variable in the box, i.e., the set of variables in $\PR_{Z \to W}$ except for $Z$, $\PR_{W \gets M}$ except for $M$, and $\PR_{W \to S}$.
        Again, the directed paths do not overlap due to the minimality of $\Gr_2$.

        To introduce $X^*, Y^*, \Gr^*$ when $I$ exists (i.e., for the cases depicted in Figures \ref{fig: I exists, no collider} and \ref{fig: I exists, one collider}), we consider the following cases.

        \begin{itemize}
            \item \textit{$\PR_{I' \to S}$ has at least one variable in $\Xb_1$:}
                In this case, let $X^*$ be the closest variable to $S$ in path $\PR_{I' \to S}$ that is in $\Xb_1$.
                Furthermore, let $\Gr^*$ be the DAG consisting of paths $\PR_{I' \to S}$ and $\PR_{I' \to Y}$.
                Therefore, DAG $\Gr^*$ is in the form of Figure $\ref{fig: no I}$, where $Z$ and $X$ in Figure \ref{fig: no I} are $I'$ and the closest variable to $I'$ in path $\PR_{I' \to S}$ that is in $\Xb_1$, respectively.
                Nevertheless, the same proof works for this case.
            \item \textit{$\PR_{I' \to S}$ has no variables in $\Xb_1$:}
                In this case, let $\Gr^*$ be $\Gr_2$ and $X^*$ be the closest variable to $I$ in path $\PR_{X\to I}$ that is in $\Xb_1$.
                Note that $X^*$ exists since $X \in \Xb_1$ and $X$ is in path $\PR_{X\to I}$.
                In $\Gr^*$, the variables in $\Xb_1$ can only be in $\PR_{X\to X^*}$ because $\PR$ and $\PR_{I' \to S}$ has no variables in $\Xb_1$.
                Hence, in $\Gr^*$, the directed paths from $\Xb \setminus \{X^{*}\}$ to $S$ pass through $X^*$, thus, Condition 1 holds.
                Moreover, the second condition holds since $S$ does not block path $(\PR_{X^* \gets X}, \PR)$ in $\Gr^*_{\underline{X^*}}$.
                To show the third condition, we note that the incoming edges of $\Xb$ are removed in $\Gr^*_{\overline{\Xb}}$.
                This shows that there exists no path between $\Xb \setminus \{X^*\}$ and $Y^*$ in $\Gr^*_{\overline{\Xb}}$, thus, Condition 3 holds.
        \end{itemize}

        In all of the cases, we introduced variables $X^* \in \Xb_1, Y^* \in \Yb$ and a subgraph $\Gr^*$ of $\Grs$ and showed that the three conditions of the lemma hold.
        This concludes the proof.
    \end{proof}